\documentclass[letterpaper]{article} 
\usepackage{aaai25}  
\usepackage{times}  
\usepackage{helvet}  
\usepackage{courier}  
\usepackage[hyphens]{url}  
\usepackage{graphicx} 
\usepackage{natbib}  
\usepackage{caption} 
\usepackage{algorithm}
\usepackage{algorithmic}

\usepackage{pdfpages}

\usepackage{amsmath}
\usepackage{amssymb}
\usepackage{mathtools}
\usepackage{amsthm}
\usepackage{multirow}
\usepackage{enumitem}
\usepackage{subcaption}
\usepackage{booktabs}
\usepackage{pdfpages}

\theoremstyle{plain}
\newtheorem{theorem}{Theorem}[section]
\newtheorem{proposition}[theorem]{Proposition}

\theoremstyle{definition}
\newtheorem{definition}[theorem]{Definition}
\newtheorem{assumption}[theorem]{Assumption}
\theoremstyle{remark}
\newtheorem{remark}[theorem]{Remark}

\newcommand{\indep}{\perp \!\!\! \perp}

\usepackage{newfloat}
\usepackage{listings}
\DeclareCaptionStyle{ruled}{labelfont=normalfont,labelsep=colon,strut=off} 
\floatstyle{ruled}
\newfloat{listing}{tb}{lst}{}
\floatname{listing}{Listing}
\title{ABC3: Active Bayesian Causal Inference with Cohn Criteria \\in Randomized Experiments}
\author {
    Taehun Cha and
    Donghun Lee\thanks{\; corresponding author}
}
\affiliations {
    Korea University\\
    \{cth127, holy\}@korea.ac.kr
}

\usepackage{bibentry}

\begin{document}

\maketitle

\begin{abstract}
In causal inference, randomized experiment is a de facto method to overcome various theoretical issues in observational study.
However, the experimental design requires expensive costs, so an efficient experimental design is necessary.
We propose ABC3, a Bayesian active learning policy for causal inference.
We show a policy minimizing an estimation error on conditional average treatment effect is equivalent to minimizing an integrated posterior variance, similar to Cohn criteria \citep{cohn1994active}.
We theoretically prove ABC3 also minimizes an imbalance between the treatment and control groups and the type 1 error probability.
Imbalance-minimizing characteristic is especially notable as several works have emphasized the importance of achieving balance.
Through extensive experiments on real-world data sets, ABC3 achieves the highest efficiency, while empirically showing the theoretical results hold.\footnote{Source codes are available on \url{https://github.com/AIML-K/ActiveBayesianCausal/}. Technical Appendix is on arXiv version.}
\end{abstract}

%

\section{Introduction}

The major goal of causal inference is to estimate the treatment effect which is a relative effect on the treatment group compared to the control group.
In randomized experiments, practitioners allocate treatment to subjects to estimate the treatment effect.
Randomized experiments free practitioners from various theoretical problems prevailing in an observational study, e.g. unmeasured confounders.
However, a randomized experiment is usually more expensive than an observational study, as a result, an efficient experiment design is desirable.

To achieve efficiency in randomized experiments, \citet{efron1971forcing} first introduced a biased-coin design, and several works tried to minimize the estimation bias by achieving a balance between treatment and control groups \citep{atkinson2014selecting}.
\citet{antognini2011covariate}  extended this to covariate-adaptive design to achieve a balance, not only between treatment-control groups but also within sampling strata.
Several recent works target the same goal using an adaptive Neyman allocation \citep{dai2023clip} or Pigeonhole design \citep{zhao2024pigeonhole}.
However, these lines of work assume the experimental subjects are \textit{given}, not actively \textit{choosable}.

Active learning is a framework where a practitioner can choose unlabeled data points and ask an oracle to label them \citep{settles2009active}.
We can further rationalize the experimental design by adopting an active learning framework.
For example, internet companies can choose a member whom they implement the A/B test, by utilizing the personal information they have gathered.
Pharmaceuticals can choose a subject based on their personal information after the applicants are gathered to save a budget.

To develop a sound active learning method for randomized experiments, (1) it should not violate the standard assumptions in causal inference.
Also, (2) the method should achieve a balance between observation and control groups to make a sound conclusion.
We introduce ABC3, a novel active learning policy for randomized experiments to remedy these issues.
Using the Gaussian process, our policy targets minimizing the error of individual treatment effect estimation.
Our contributions are three folds:
\begin{itemize}
    \item We theoretically show minimizing the estimation error on individual treatment effect is equivalent to minimizing the integrated predictive variance in a Bayesian sense.
    \item ABC3, the policy minimizing the variance, theoretically minimizes the imbalance between treatment and control groups and the type 1 error rate.
    \item With extensive experiments, we empirically verify ABC3 outperforms other methods while showing theoretical properties hold.
\end{itemize}

After examining related works in Section 2, we formalize our problem in Section 3.
Then we introduce ABC3 and its theoretical properties in Section 4.
We empirically verify the performance and properties of ABC3 in Section 5.
Then we conclude our paper with several discussions and limitations in Section 6 and 7.

\section{Related Works}

There are several works exploring an active learning policy for observational data (\citet{sundin2019active}, \citet{jesson2021causalbald}, and \citet{toth2022active}).
Especially, \citet{sundin2019active} proposed an active learning policy for decision making, when treatment-control groups are imbalanced.
They theoretically showed the imbalance can result in a type S error, where a practitioner estimates the treatment effect with a different sign.
Likewise, \citet{shalit2017estimating} showed that the generalization error is bounded by the imbalance when estimating the treatment effect.
However, their work focused on obtaining a balanced representation from observational data, not a randomized experiment setting.

For randomized experiments, \citet{deng2011active} suggests an active learning policy sampling a point with the highest predictive variance which is similar to Mackay's criteria \citep{mackay1992information}.
\citet{zhu2024integrating} also proposed Mackay's criteria-like method under network interference structure.
\citet{song2023ace} suggested ACE which targets maximizing the covariance between observed and test data sets.
We will compare the effectiveness of our policy with these policies.

Sample-constrained causal inference setting shares a similar goal with active learning: achieving lower estimation error with fewer samples.
\citet{addanki2022sample} and \citet{ghadiri2023finite} proposed an efficient sampling and estimation method in a randomized experiment setting.
\citet{harshaw2023balancing} suggested a sampling method balancing covariates.
However, their work assumes a linear relationship between covariates and potential outcomes, which is vulnerable in real-world scenarios.
We will also compare the effectiveness of these policies.

\section{Problem Formulation}

Let $X \in \mathcal{X}, Y \in \mathcal{Y}$ and $A \in \{0, 1\}$ be random variables.
$X$ is a covariate representing each subject, $Y$ is an outcome, and $A$ represents a binary treatment.
Following the Neyman-Rubin causal model \citep{rubin1974estimating}, additionally define $Y^0$ and $Y^1$, potential outcomes for either control or treatment.
Unlike the usual supervised learning settings, a practitioner can observe only one of $Y^0$ and $Y^1$ in a causal inference setting.
$x, y, a, y^0$ and $y^1$ denotes the realizations of each random variable.

Let $D_\Omega = \{(x_i, y^0_i, y^1_i)\}_{i=1}^{N}$ be a subject pool with covariate information $x_i$ and potential outcomes $y^0_i, y^1_i  \in \mathbb{R}$.
Let $D_t^0 = \{(x_i, y^0_i)\}_{i \in I_t^0}$ be an observed control group data set at time $t$ with an index set $I_t^0$.
Likewise, define $D_t^1 = \{(x_i, y^1_i)\}_{i \in I_t^1}$ for a treatment group.
Let $X_{\Omega}, X_t^0$ and $X_t^1$ be sets of $x$s in each data set, $D_{\Omega}, D_t^0$ and $D_t^1$.

Our quantity of interest is the \textit{conditional average treatment effect} (\textit{CATE}), $CATE(x) = \mathbb{E} \left[ Y^1 - Y^0 | X=x \right]$ for each subject $x$.
Then we can train an estimator $\hat{CATE}_t(x) = \hat{y}^1_t(x) - \hat{y}^0_t(x)$, where $\hat{y}^a_t$s are regressors trained on each data set $D_t^a$, $a \in \{0, 1\}$.

To evaluate the trained estimator, we use a standard metric, \textit{expected precision in estimation of heterogeneous effect} (\textit{PEHE}, \citet{hill2011bayesian}),

\begin{equation*}
    \epsilon_{PEHE}(\hat{CATE}_t) = \int_\mathcal{X} (\hat{CATE}_t(x) - CATE(x))^2 d\mathbb{P}(x),
\end{equation*}
where $\mathbb{P}$ is a probability distribution over whole covariates.
In the usual case, we can treat $\mathbb{P}$ as a discrete distribution corresponding to the covariates in $D_\Omega$.
Then we can write $\mathbb{P} = \frac{1}{N} \Sigma_{i=1}^{N} \delta_{x_i}$, where $\delta$ is Dirac-delta function.

Meanwhile, in Bayesian statistics, we do not hypothesize the existence of population parameters, e.g. $CATE(x)$.
Instead, we update our belief with observed data points.
To derive a Bayesian policy, we need to define the optimization target in the Bayesian sense.

\begin{definition}
    \begin{equation*}
        \epsilon_{PEHE}^{\Omega}(\hat{CATE}_t) = \int_\mathcal{X} (\hat{CATE}_t(x) - \hat{CATE}_{\Omega}(x))^2 d\mathbb{P}(x)
    \end{equation*}
    where $\hat{CATE}_{\Omega}(x)=\hat{y}^1_{\Omega}(x) - \hat{y}^0_{\Omega}(x)$, $\hat{y}^a_{\Omega}$s are regressors trained on whole data set $D_{\Omega}$.
\end{definition}

Intuitively, $\hat{CATE}_{\Omega}$ represents the best oracle estimation observing all factual and counterfactual data points.
At any $t$, our active learning problem is to choose an optimal (yet unobserved) subject $x^*$ and its treatment $a^*$  which can achieve the lowest expected $\epsilon_{PEHE}^{\Omega}(\hat{CATE}_{t+1})$, i.e.

\begin{align*}
    x^*, a^* = & \text{arg min}_{x_{t+1} \in X_{\Omega} \setminus (X^1_t \bigcup X^0_t), a_{t+1} \in \{0, 1\}} \\ 
    & \mathbb{E}_{t+1} \left[\epsilon_{PEHE}^{\Omega}(\hat{CATE}_{t+1})\right],
\end{align*}
where $\mathbb{E}_t \left[ \cdot \right] = \mathbb{E} \left[ \cdot | \mathcal{F}_t \right]$, $\mathcal{F}_t$ is a filtration at $t$.
If $x^* = x_j$ and $a^*=a$, we observe $y^a_j$ and append it to the observed data set $D^a_{t+1} = D^a_t \bigcup \{(x_j, y^a_j)\}$ to be used at $t+1$.

While optimizing the expected error, an active learning policy should not violate the standard assumptions for causal inference.
A standard randomized experiment in causal inference requires several assumptions to identify the causal effects \citep{rubin1974estimating}.

\begin{assumption} (Consistency)
    $Y = AY^1 + (1-A)Y^0$
\end{assumption}

\begin{assumption} (Positivity)
    $\mathbb{P}(A=a|X=x) > 0, \forall x$
\end{assumption}

\begin{assumption} (Randomization)
    $ (Y^0, Y^1) \indep A | X$
\end{assumption}

Assumption 3.2. assumes that the observed outcome $Y$ under treatment $A$ is equivalent to its potential outcome $Y^a$.
Assumption 3.3. is required to well-define the conditional expectation.
Assumption 3.4. implies that \textbf{the treatment $A$ should be assigned independently from the potential outcome values}.
The assumptions guarantee an unbiased estimation of CATE, i.e. $CATE(x)=\mathbb{E} \left[ Y^1 - Y^0 | X=x \right]= \mathbb{E} \left[ \mathbb{E} \left[ Y | A=1, X=x \right] - \mathbb{E} \left[ Y | A=0, X=x \right] \right]$.

Assumption 3.4. is crucial in an active learning setting since several active learning algorithms consider $y$ values when querying $x$ and $a$.
For instance, \citet{song2023ace} introduced ACE-UCB, a bandit-like policy that targets a subject with the highest individual treatment effect.
However, ACE-UCB exploits previously observed outcomes when choosing a subject and treatment.
As a result, ACE-UCB can generate confounding between treatment and potential outcomes and may violate Assumption 3.4.
So querying $x$ and $a$ without considering $y$s is crucial when adopting an active learning framework for a causal inference.

\section{Proposed Method}

\subsection{ABC3: Active Bayesian Causal Inference with Cohn Criteria}

Gaussian process ($\mathcal{GP}$, \citet{rasmussen2006gaussian}) is a non-parametric machine learning model based on Bayesian statistics.
It is a powerful tool as it allows flexible function estimation depending on a pre-defined kernel function.
Set priors on $f \sim \mathcal{GP}(m(x), k(x,x'))$, where $m(x)$ is a mean prior and $k$ is a kernel function.
Assume a data set $D=\{(x_i, y_i)\}_{i=1}^N$ with noisy observation $y_i=Y(x_i)+\epsilon_i$, $\epsilon_i \sim \mathcal{N}(0, \sigma^2_\epsilon)$.
We can obtain a posterior distribution of $f(x^*)$ given data set $D$ by computing $f(x^*) | D$ as
\begin{align*}
    & f(x^*) | D \sim \mathcal{N}(m'(x^*), \sigma^2(x^*)) \text{ where}\\
    & m'(x^*) = m(x^*) + \mathbf{k}_*(x^*)^T \left[ \mathbb{K} + \sigma^2_\epsilon \mathbf{I} \right]^{-1} \mathbf{y}, \\
    & \text{cov}(x, x^*) = k(x, x^*) - \mathbf{k}_*(x)^T \left[ \mathbb{K} + \sigma^2_\epsilon \mathbf{I} \right]^{-1} \mathbf{k}_*(x^*), \\
    & \sigma^2(x^*) = \text{cov}(x^*, x^*), \\
    & \mathbf{k}_*(x^*) = \left[ k(x_i, x^*) \right]_{i=1}^N, \mathbb{K} = \left[ k(x_i, x_j) \right]_{i,j=1}^N, \\
    & \mathbf{y} = \left[ y_i \right]_{i=1}^N.
\end{align*}

We can observe that the posterior variance $\sigma^2(x^*)$ does not depend on $\mathbf{y}$.
Adopting the notations from above, we define $\mathbf{k}_{t,*}^a, \mathbb{K}^a_t, \mathbf{y}^a_t$ and $\text{cov}^a_t$ with observed data set $D^a_t$ at time $t$.
We also assume zero-mean prior, i.e. $m(x)=0, \forall x$.

The following theorem identifies our active learning policy with only posterior variance terms.

\begin{theorem}
    Assume $|k(x, x')| < \infty$ and $|y^a_i| < \infty$ for all $x, x' \in \mathcal{X}$ and $a, i$, as a result, $\epsilon^{\Omega}_{PEHE}(\hat{CATE}_t) < \infty, \forall t$.
    Let our estimator $\hat{CATE}_t(x) = \hat{y}^1_t(x) - \hat{y}^0_t(x)$, where $\hat{y}^a_t(x)=\mathbb{E}_t\left[ Y^a(x) \right]$ is a mean posterior distribution of gaussian process $Y^a$ trained on data set $D_t^a$. 
    Assume two Gaussian processes $Y^1$ and $Y^0$ are independent.
    Then 
    \begin{align*}
     \text{arg min}_{x_{t+1},a_{t+1}} & \mathbb{E}_{t+1} \left[\epsilon_{PEHE}^{\Omega} (\hat{CATE}_{t+1})\right] =\\
    \text{arg min}_{x_{t+1},a_{t+1}} & \int_\mathcal{X} \mathbb{V}_{t+1} \left[ Y^1(x) \right] + \mathbb{V}_{t+1} \left[Y^0(x) \right] d\mathbb{P}(x)
    \end{align*}
\end{theorem}

Proof of Theorem 4.1. is in Appendix A.1.
Theorem 4.1. states that minimizing the error on CATE estimation is equivalent to minimizing the integrated posterior variance of the estimator.
In active learning literature, the active learning policy minimizing the integrated predictive variance is called \textit{Active Learning Cohn} (\citet{cohn1994active} and \citet{gramacy2020surrogates}).
\citet{seo2000gaussian} proposed Active Learning Cohn policy utilizing the Gaussian process in a usual supervised learning setting.
We extend this line of work to causal inference and name our policy after his name, \textbf{ABC3: Active Bayesian Causal Inference with Cohn Criteria}.




Computing the inverse of all hypothetical covariance matrix $\mathbb{K}^a_{t+1}$ for all $x$ is computationally infeasible.
However, we can efficiently find the minimizer as $\mathbb{K}^a_t$ is a principal submatrix of $\mathbb{K}^a_{t+1}$.

\begin{proposition}
    \begin{align*}
        & \text{arg min}_{x_{t+1},a_{t+1}} \int_\mathcal{X} \mathbb{V}_{t+1} \left[ Y^1(x) \right] + \mathbb{V}_{t+1} \left[ Y^0(x) \right] d\mathbb{P}(x) \\
        &= \text{arg max}_{x_{t+1},a}
    \end{align*}
    \begin{align}
        & \frac{ \int_\mathcal{X} \left[ (\mathbf{\tilde{k}}^a_{t+1})^T \left[ \mathbb{K}^a_t + \sigma^2_\epsilon \mathbf{I} \right]^{-1} \mathbf{k}^a_{t,*}(x) - k(x_{t+1}, x)\right]^2 d\mathbb{P}(x) }{k(x_{t+1},x_{t+1}) + \sigma^2_\epsilon - (\mathbf{\tilde{k}}^a_{t+1})^T  \left[ \mathbb{K}^a_t + \sigma^2_\epsilon \mathbf{I} \right]^{-1} \mathbf{\tilde{k}}^a_{t+1}}
    \end{align}
    where $\mathbf{\tilde{k}}^a_{t+1} = \left[k(x_i, x_{t+1})\right]_{i \in I^a_t}$
\end{proposition}

Proof of Proposition 4.2. is in Appendix A.2.
From the proposition, we successfully eliminate the dependency on $\mathbb{K}^a_{t+1}$, and there is no need to compute an inverse matrix for every $x$.
We summarize our policy in Algorithm 1.

\begin{algorithm}[h]
\caption{ABC3}
\textbf{Input}: Current time step $t$, whole covariate set $X_{\Omega}$, covariates distribution $\mathbb{P}$, previous observations $X^1_t$ and $X^0_t$, kernel $k$, noise parameter $\sigma_\epsilon$\\
\textbf{Output}: $x_{t+1}, a_{t+1}$
\begin{algorithmic}[1] 
\STATE $\mathcal{V}_0, \mathcal{V}_1 = \phi, \phi$
\FOR{$x \in X_{\Omega} \setminus (X^0_t \bigcup X^1_t)$}
\STATE Compute $\mathbf{\tilde{k}}^1_{t+1}$ and $\mathbf{\tilde{k}}^0_{t+1}$ assuming $x_{t+1}= x$
\STATE $v^0, v^1=$ Equation (1) for each $a \in \{0,1\}$
\STATE $\mathcal{V}_0 = \mathcal{V}_0 \bigcup \{v^0\}, \mathcal{V}_1 = \mathcal{V}_1 \bigcup \{v^1\}$
\ENDFOR
\STATE $i, a_{t+1} = \text{arg max} (\mathcal{V}_0 || \mathcal{V}_1)$
\STATE $x_{t+1} = X_{\Omega} \setminus (X_t^0 \bigcup X_t^1)\left[ i \right]$
\STATE \textbf{return} $x_{t+1}, a_{t+1}$
\end{algorithmic}
\end{algorithm}

\subsection{Theoretical Analysis}

\subsubsection{Balancing Treatment-Control Groups}

Unlike usual supervised learning, causal inference requires precise estimation of both functions for treatment and control groups.
As a result, the balance between the two groups is crucial to obtain a sound estimation.
For example, consider a study on the causal effect of online lectures.
Assume our treatment group is concentrated on undergraduate students while our control group is concentrated on graduate students.
Then it would be difficult to make a sound conclusion with statistical tools, as the two groups are highly imbalanced.

Several researchers theoretically analyzed the effect of the imbalance on causal inference.
\citet{shalit2017estimating} showed that the generalization error on CATE is upper bounded by the imbalance.
\citet{sundin2019active} defined Type S Error, assigning a different sign (+ or -) to CATE, and showed the probability of Type S Error is bounded by the imbalance.
Both works utilized \textit{Maximum Mean Discrepancy} (MMD, \citet{gretton2012kernel}) to quantify and measure the imbalance.

\begin{definition}
    \begin{align*}
        \text{MMD}(P, Q,\mathcal{F}) &= \\
        &\text{sup}_{f \in \mathcal{F}} \mathbb{E}_{x \sim P(x)} \left[ f(x) \right] - \mathbb{E}_{y \sim Q(y)} \left[ f(y) \right]
    \end{align*}
    where $\mathcal{F}$ is a unit ball of Reproducing Kernel Hilbert Space (RKHS) induced by a kernel $k(\cdot, \cdot)$.
\end{definition}

\citet{gretton2012kernel} showed that the MMD is equivalent to the distance between mean embeddings in RKHS, $\mu_P$ and $\mu_Q$, and can be computed with the kernel function $k$.

\begin{remark}
    \begin{align*}
        \text{MMD}(P, Q, \mathcal{F})^2 &= || \mu_P - \mu_Q ||^2 \\
        = \mathbb{E}_{x,x' \sim P(x)} &\left[ k(x, x') \right] + \mathbb{E}_{y,y' \sim Q(y)} \left[ k(y, y') \right] \\
        & - 2 \mathbb{E}_{x \sim P(x), y \sim Q(y)} \left[ k(x, y) \right]
    \end{align*}
\end{remark}

We show our active learning policy approximately minimizes the upper bound of MMD.

Assume the probability over whole covariates is a discrete distribution corresponding to the covariates in $D_\Omega$, i.e. $\mathbb{P} = \frac{1}{N} \Sigma_{i=1}^{N} \delta_{x_i}$.
For notational convenience, we assume the noise-less observation case $\sigma^2_{\epsilon}=0$, but we can easily extend the result to the noisy case.
Let $\mathbb{P}_t^a = \frac{1}{|I_t^a|} \Sigma_{i \in I_t^a} \delta_{x_i},  a \in \{0, 1\}$ be an empirical distribution for treatment-control group up to time $t$.
Then we obtain the following theorem.

\begin{theorem} 
    Assume $|k(x, x')| < \infty$ and $|y^a_i| < \infty$ for all $x, x' \in \mathcal{X}$ and $a, i$. 
    Let $\lambda^*$ be a maximum eigenvalue of $\mathbb{K}_{\Omega}$, i.e. covariance matrix of whole covariates.
    Let $M = \frac{1}{N^2} \Sigma_{i,j=1}^{N} k(x_i,x_j)$.
    Define functions $\delta^*(I_n)=\frac{1}{n} \Sigma_{i \in I_n} \int_{\mathcal{X}} k(x_i, x) d\mathbb{P}(x)$, and $\epsilon^*(I_n) = M - \frac{1}{n^2} \Sigma_{i,j \in I_n} k(x_i, x_j)$, where $I_n \subset \{1,...,N\}$ is an n-element index set.
    Assume  $\epsilon^*(I_n) \leq 2 \delta^*(I_n), \forall I_n$.
    Then
    \begin{align*}
        MMD(\mathbb{P}^1_t, & \mathbb{P}^0_t, \mathcal{F})^2 \leq  4 \frac{\lambda^*}{|I^1_t|} + 4 \frac{\lambda^*}{|I^0_t|} \\
        & + 2 \int_\mathcal{X} \mathbb{V}_t \left[ Y^1(x) \right] + \mathbb{V}_t \left[ Y^0(x) \right] d\mathbb{P}(x)
    \end{align*}
\end{theorem}

Proof of Theorem 4.5. is in Appendix A.3.
We can observe that the first two terms decrease as we observe more subjects.
The third term is our exact optimization target as introduced in Theorem 4.1.
Empirical achievability and intuitive meaning of the assumption are covered in Section 5.6 and Appendix B.

By combining this result with the previous theoretical analysis (\citet{shalit2017estimating} and \citet{sundin2019active}), our active learning policy minimizes the upper bounds of both generalization error and type S error.

\subsubsection{Type 1 Error Minimization}

The precise estimation of causal effect is an important object of causal inference.
However, \textbf{testing the existence of causal effect} is another key component.
R. A. Fisher first advocated the randomized experiment to test the existence of causal effects \citep{fisher1970statistical}.
He proposed \textit{Fisher's sharp null hypothesis}, where $y^1_i=y^0_i, \forall i=1, ... ,N$.
A randomized experiment method should not reject the null hypothesis if it finds no causal effect.
Type 1 error occurs when we reject a null hypothesis while it is true.

As active learning is a sequential procedure, a practitioner may want to test the statistical significance of the treatment effect sequentially.
However, \citet{ham2023design} stated that applying statistical tests sequentially can result in a high type 1 error probability.
An active learning policy should avoid the type 1 error to make a sound conclusion.

Here we present our theoretical result that our policy minimizes the upper bound of the type 1 error probability.
First, we define the type 1 error.

\begin{definition} (Type 1 Error at $t$) 
    Under Fisher's Sharp null hypothesis, as it implies $\text{CATE}(x)=0, \forall x$,
    \begin{equation*}
        \mathbb{P}_t\left[ \text{Type 1 Error}(x) \right] = \mathbb{P}_t\left[ |Y^1(x) - Y^0(x)| > \alpha \right],
    \end{equation*}
    where $\alpha$ is a decision threshold.
\end{definition}

Then we can show that our policy minimizes the upper bound of the type 1 error rate over the whole $x$.
Proof of Theorem 4.7. is in Appendix A.4.

\begin{theorem}
    Under Fisher's Sharp null hypothesis, $\text{arg min}_{x_{t+1}, a_{t+1}} \int_\mathcal{X} \mathbb{V}_{t+1} \left[ Y^1(x) \right] + \mathbb{V}_{t+1} \left[ Y^0(x) \right] d\mathbb{P}(x)$ also minimizes the upper bound of the $\int_{\mathcal{X}} \mathbb{P}_{t+1} \left[ \text{Type 1 Error}(x) \right] d\mathbb{P}(x)$
\end{theorem}

\section{Experiments}
In this section, we empirically analyze the theoretical results introduced in Section 4.
For the comparison, we utilize \textbf{IHDP} (\citet{brooks1992effects} and \citet{hill2011bayesian}), \textbf{Boston} \citep{harrison197881hedonic}, \textbf{ACIC} \citep{gruber2019acic}, and \textbf{Lalonde} \citep{lalonde1986evaluating} data sets.
IHDP and ACIC data sets are semi-real data sets where counterfactual outcomes are simulated.
Boston and Lalonde data sets do not contain counterfactual outcomes.
Following \citet{addanki2022sample}, we set $Y^0=Y^1$ to simulate the null hypothesis circumstance (i.e. $CATE(x)=0, \forall x$).
All data sets have continuous outcome values as we assume the potential outcome follows a Gaussian process.
More detailed explanations of data sets are in Appendix C.



We randomly divide each data set in half for every trial to construct train and test data sets.
Each active learning policy selects what to observe from the train data set and regressors are trained on the observed train data set.
Policies are evaluated on the test data set for each pre-defined time step.
We utilize the following baseline policies.

\begin{figure*}[t]
    \centering
    \includegraphics[width=0.4\linewidth]{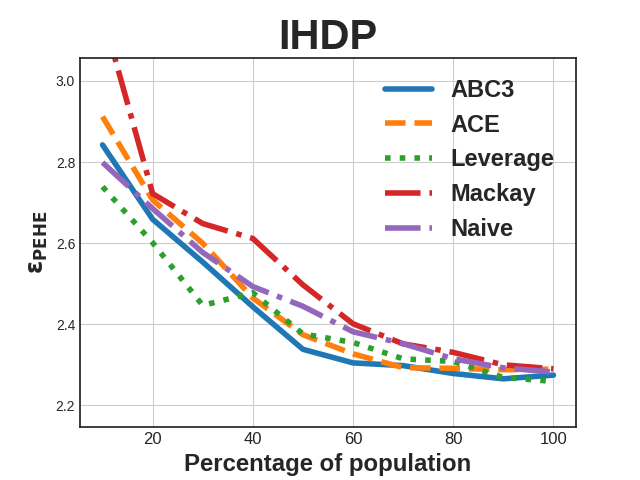}
    \includegraphics[width=0.4\linewidth]{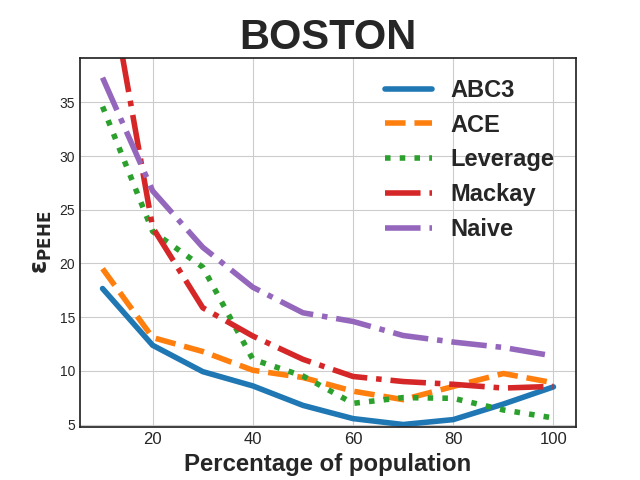}
    \includegraphics[width=0.4\linewidth]{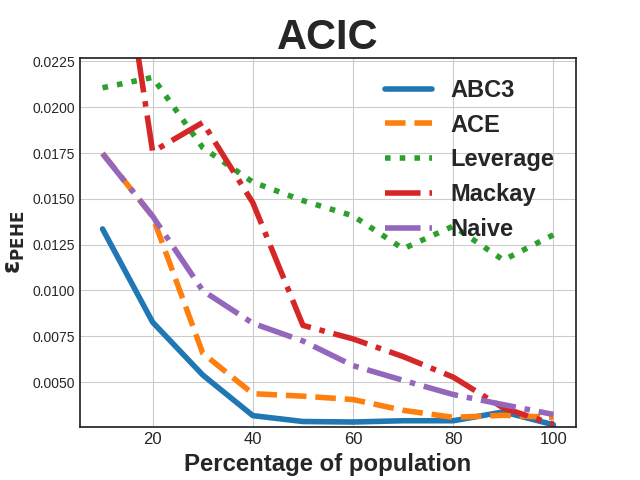}
    \includegraphics[width=0.4\linewidth]{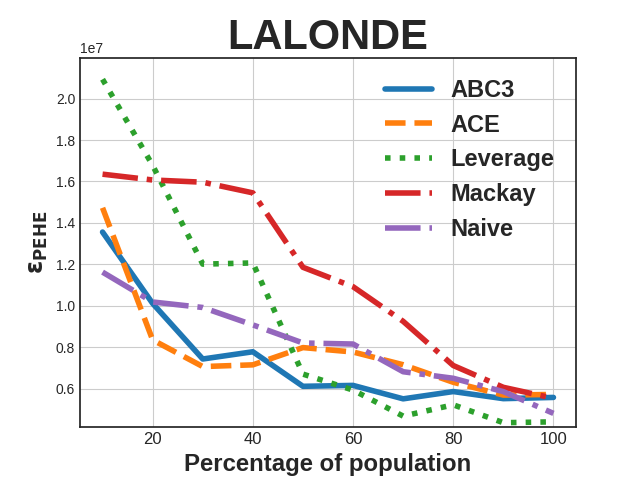}
    \caption{Mean of $\epsilon_{PEHE}$. $x$-axis represents the observed percentage of the population. We measure $\epsilon_{PEHE}$ for every 10\% observation.}
\end{figure*}

\begin{itemize}
    \item \textbf{Naive}: This policy randomly selects $x_{t+1}$ from the subject pool and then decides treatment $a_{t+1}$ with Bernoulli distribution with probability 0.5.
    \item \textbf{Mackay} \citep{mackay1992information}: This policy selects $x_{t+1}, a_{t+1} = \text{arg max}_{x, a} \mathbb{V}_t \left[ Y^a(x) \right]$, i.e. chooses the most uncertain point at $t$.
    \item \textbf{Leverage} \citep{addanki2022sample}: Under linearity assumption between covariates and outcomes, this policy exploits leverage score and shows theoretical guarantees on nearly optimal RMSE on CATE estimation. Unlike other policies, this policy is not sequential, as a result, the observed points for $t$ and $t+1$ can be different.
    \item \textbf{ACE} \citep{song2023ace}: Unlike other policies, this policy assumes access to the test data set and select $x_{t+1}, a_{t+1} = \text{arg max}_{x,a} \left[ \frac{1}{|n_{test}|} \Sigma_{i=1}^{|n_{test}|} \text{cov}^a_t\left[ x^{test}_i, x \right] \right]^2 / \mathbb{V}_t \left[ Y^a(x) \right]$. It maximizes the covariance between observed and test data sets while minimizing the predictive variance.
\end{itemize}

After selecting $x_{t+1}$ and $a_{t+1}$, all policies fit the Gaussian process for regression.
We apply feature-wise normalization and $y$-standardization for all regressors. (except Leverage, which requires item-wise normalization)
We fit two Gaussian process models with Constant Kernel * Radial Basis Function (RBF) Kernel + White Kernel.
We optimize the kernel hyperparameters using \textit{scikit-learn} package.

All the uncertainty-aware policies (ABC3, Mackay and ACE) use the Gaussian process to quantify the uncertainty.
For the uncertainty-quantifying kernels, we use RBF kernel with length scale 1.0 with $\sigma^2_{\epsilon}=1$.
(We check the hyperparameter sensitivity in Section 5.4.)

\subsection{Minimizing Error}

We measure the $\epsilon_{PEHE}$ (not $\epsilon_{PEHE}^{\Omega}$) when observing every 10\% of population.
We run 100 experiments for every data set.
We mark the mean of measured $\epsilon_{PEHE}$.
The results are in Figure 1.
Appendix D presents the standard deviation of measured $\epsilon_{PEHE}$.

ABC3 shows the best performance, i.e. the lowest $\epsilon_{PEHE}$ for most time steps.
We can verify ABC3 succesfully minimizes the population $\epsilon_{PEHE}$, though optimization target of ABC3 is $\epsilon_{PEHE}^{\Omega}$.
In most cases, when ABC3 observes only half of the population, it achieves $\epsilon_{PEHE}$ level which is achieved with full observation by other policies.
Especially for Boston, after 20\%, ABC3 achieves $\epsilon_{PEHE}$ which Naive policy cannot achieve even with full observation.


ACE policy shows comparable results to ABC3.
ACE slightly outperforms ABC3 for a 20-40\% interval of Lalonde data set.
However, ABC3 outperforms ACE in most cases, though ABC3 has no access to the test data set, unlike ACE.

Leverage temporarily outperforms ABC3 for the beginning part of IHDP and the last part of Lalonde.
However, for ACIC, Leverage significantly underperforms other policies.
The result may imply the vulnerability of linearity assumption in real-world data sets.

Mackay underperforms even Naive policy most times.
It is interesting as Mackay utilizes the same uncertainty information as ABC3.
The result may imply the importance of \textit{proper} utilization of the same information.

In summary, ABC3 is a promising Bayesian active learning policy, which efficiently and robustly achieves the best performance among the others.

\subsection{Balancing Treatment-Control Groups}

Theorem 4.5 states that ABC3 theoretically minimizes the maximum mean discrepancy between treatment and control groups.
A policy can benefit by minimizing MMD from several theoretical aspects, as introduced in Section 4.2.
We empirically verify the property.
We measure MMD between $\mathbb{P}^1_t$ and $\mathbb{P}^0_t$ for every 10\% observation and applied the same setting as the previous section.
We report the mean and standard deviation in Figure 2.

\begin{figure}[th]
    \centering
    \includegraphics[width=0.24\linewidth]{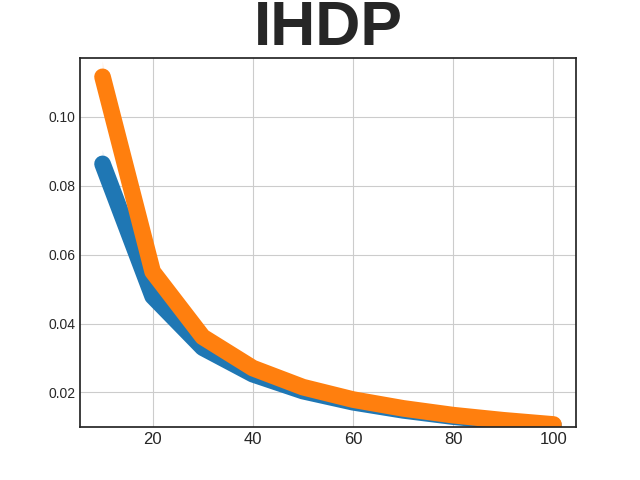}
    \includegraphics[width=0.24\linewidth]{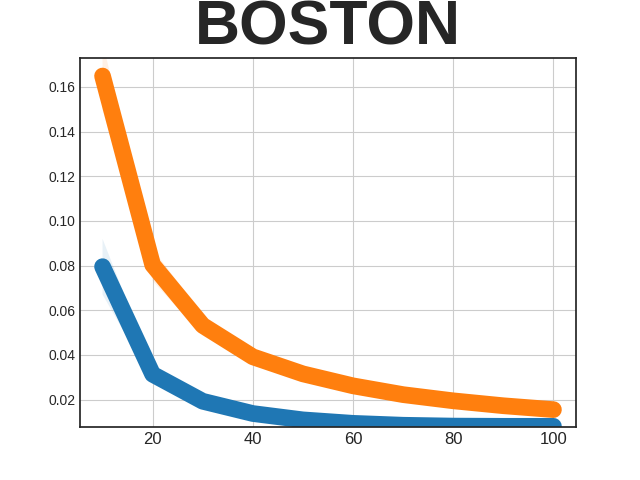}
    \includegraphics[width=0.24\linewidth]{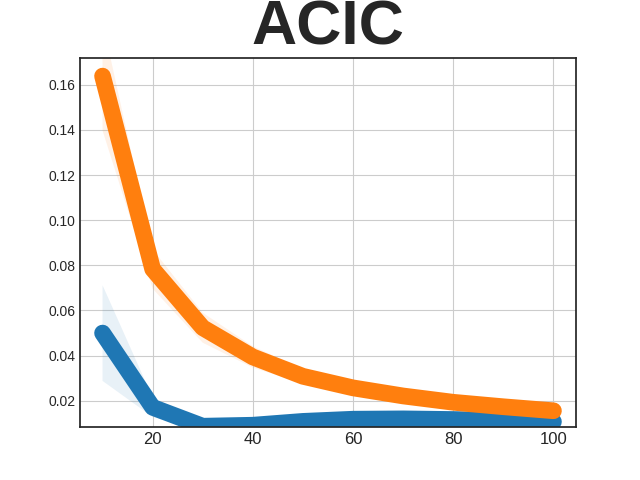}
    \includegraphics[width=0.24\linewidth]{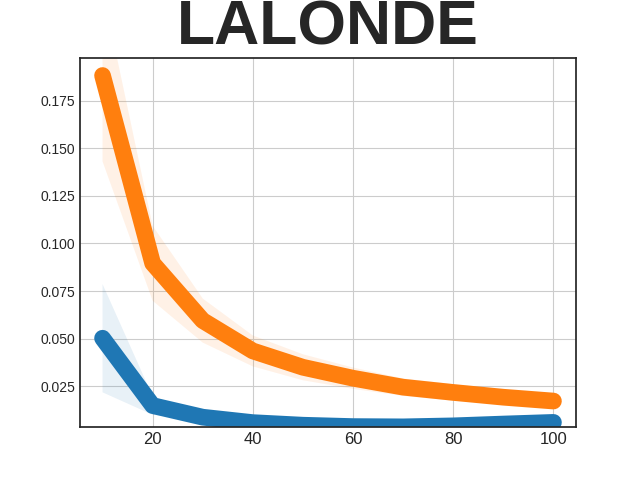}
    \caption{Mean and standard deviation of MMD between observed treatment and control groups, $\mathbb{P}^1_t$ and $\mathbb{P}^0_t$. The blue line is for ABC3, and the orange line is for Naive policy. $x$-axis is for the sampled ratio and $y$-axis is MMD.}
\end{figure}

ABC3 achieves substantially lower MMD compared to Naive policy on all data sets.
As noted in Theorem 4.5, the upper bound of MMD is minimized as the number of observations increases.
As a result, Naive policy also shows a decrease in MMD as time proceeds.
However, the MMD gap between ABC3 and Naive is significant at the beginning stage of experiments.
The gap is especially large for ACIC and Lalonde data sets.
The result empirically supports Theorem 4.5. holds, and ABC3 achieves a balance between treatment-control groups.

\subsection{Minimizing Type 1 Error}

Theorem 4.7. states ABC3 minimizes the upper bound of the integrated type 1 error probability, $\int_{\mathcal{X}} \mathbb{P}_{t+1} \left[ \text{Type 1 Error}(x) \right] d\mathbb{P}(x)$.
Here we verify the property empirically.

We use Boston and Lalonde data sets which assume no treatment effect, i.e. $Y^0=Y^1$.
To measure the type 1 error rate, we compute a mean and standard deviation of $\hat{CATE}(x)$ for every $x$ in the test data set.
Then we implement the Z test with a significance level of 5\%. (i.e. $\alpha=1.96$).
Type 1 error occurs when the absolute value of the Z-statistic is bigger than $\alpha$.
We compute the percentage of $x$ where the type 1 error occurs.
The result is in Figure 3.

\begin{figure}[th]
    \centering
    \includegraphics[width=0.495\linewidth]{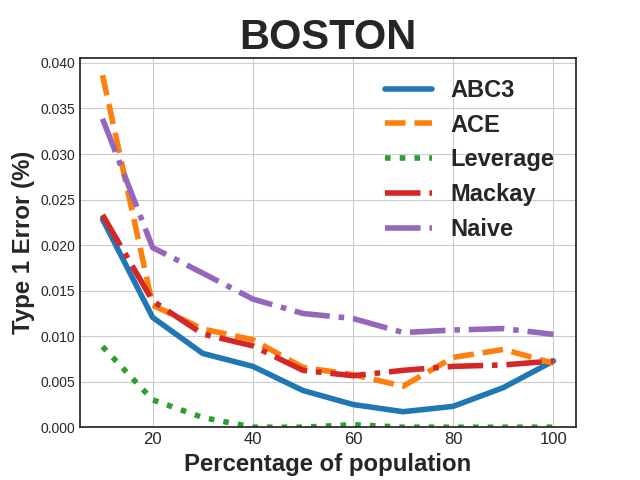}
    \includegraphics[width=0.495\linewidth]{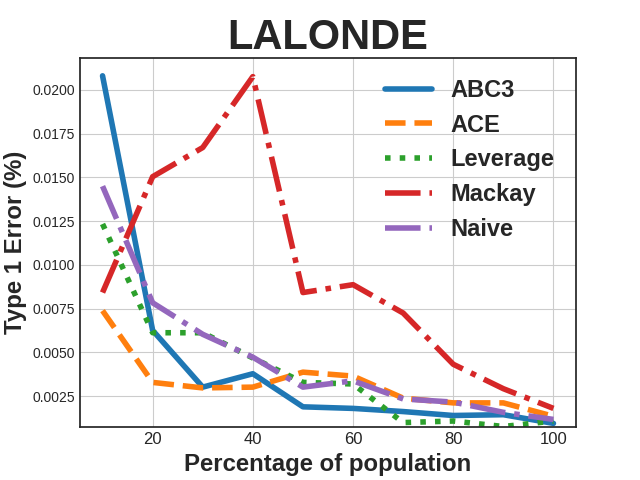}
    \caption{Mean of the type 1 error rate.}
\end{figure}

As time proceeds, ABC3 achieves a lower type 1 error rate, as expected in Theorem 4.7.
ABC3 consistently shows a lower Type 1 error rate than Naive.
However, Mackay shows the highest error rate on Lalonde.
The result may again imply the importance of \textit{proper} utilization of the uncertainty information.

Meanwhile, Leverage shows a significantly lower error rate for Boston.
This aspect was not presented in the original paper \citep{addanki2022sample}.
The result may imply the strong linearity in Boston data set and the power of the method when the linearity assumption holds.

\subsection{Hyperparameter Sensitivity}

To implement ABC3, we need to determine uncertainty-quantifying kernel, kernel parameters, and observation error $\sigma^a_{\epsilon}$ as hyperparameters.
As usual machine learning models utilizing the kernel method, selecting an appropriate kernel and parameters is crucial for obtaining precise estimation.
We present hyperparameter sensitivity analysis for ABC3.

For the kernel, we test \textbf{RBF} kernel (utilized throughout this paper), \textbf{Matern} kernel, and Exp-Sine-Squared kernel (\textbf{Sine}).
RBF kernel has one parameter, lengthscale, which determines how `local' the output function would be.
The smaller the lengthscale, the more local and wiggler the resulting function is.
Matern kernel is a generalization of RBF kernel and has two parameters, lengthscale, and smoothness.
Sine kernel assumes that our data shows a periodic pattern.
it has two parameters, lengthscale, and periodicity.
Here we test only lengthscale, with setting periodicity as 1.
We present $\epsilon_{PEHE}$ for each setting in Figure 4.

\begin{figure}[th]
    \centering
    \includegraphics[width=0.495\linewidth]{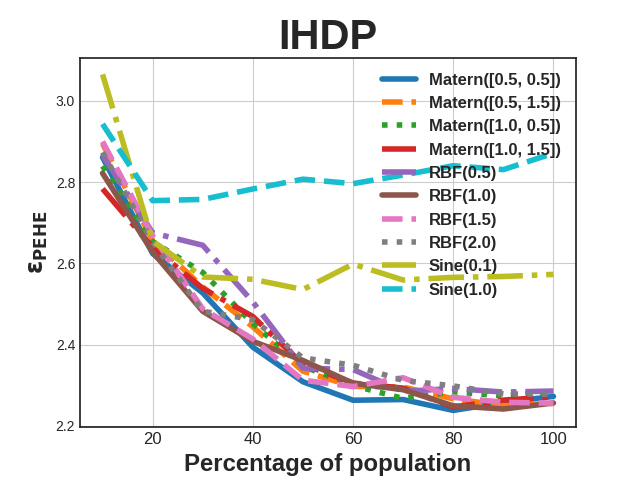}
    \includegraphics[width=0.495\linewidth]{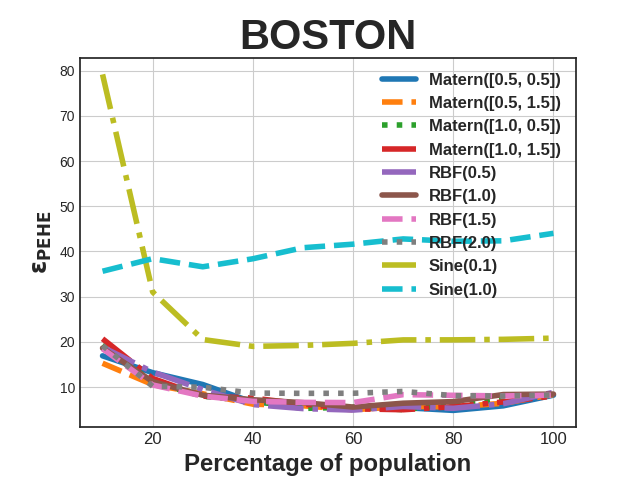}
    \includegraphics[width=0.495\linewidth]{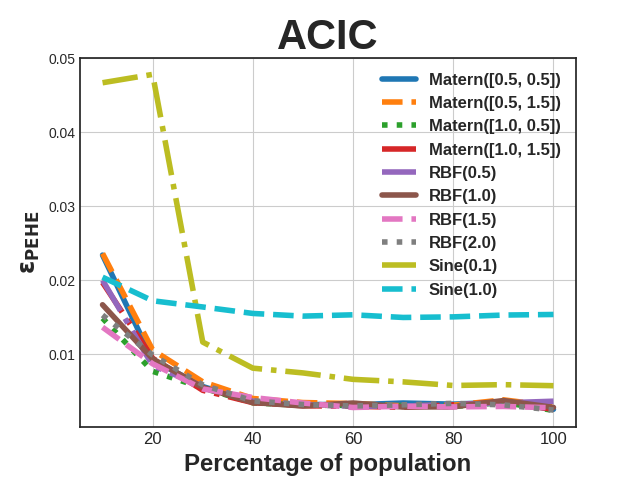}
    \includegraphics[width=0.495\linewidth]{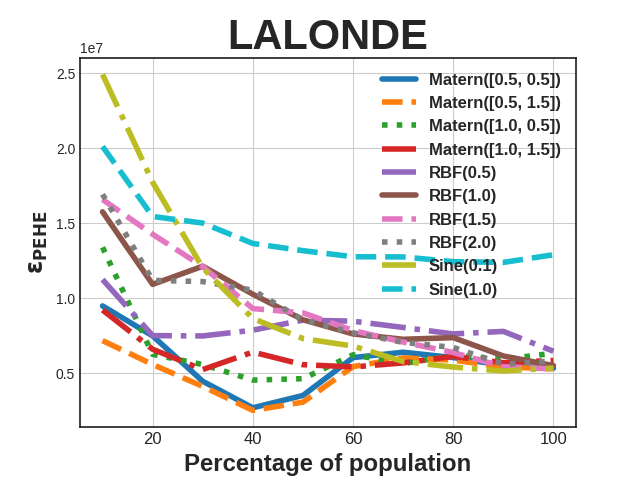}
    \caption{$\epsilon_{PEHE}$ for every kernel and kernel parameter. The numbers in parenthesis are kernel parameters.}
\end{figure}

For IHDP, ACIC, and Boston, Matern and RBF kernels show robust performance along different kernel parameters.
However, Sine kernel significantly deteriorates in the three data sets.
The result implies no (or week) periodicity in the three data sets.
Meanwhile, Sine kernel with a length scale of 0.1 shows the best performance on Lalonde data set.
It may imply a periodicity in the Lalonde data set.
The result gives a similar lesson: selecting an appropriate kernel and parameters is crucial for obtaining precise estimation.
Overall, RBF and Matern kernels with `reasonable' parameters are safe options for general data sets.

\begin{figure}[h]
    \centering
    \includegraphics[width=0.495\linewidth]{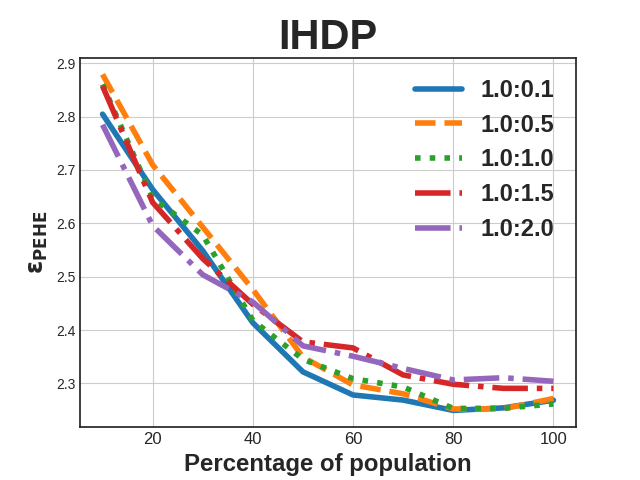}
    \includegraphics[width=0.495\linewidth]{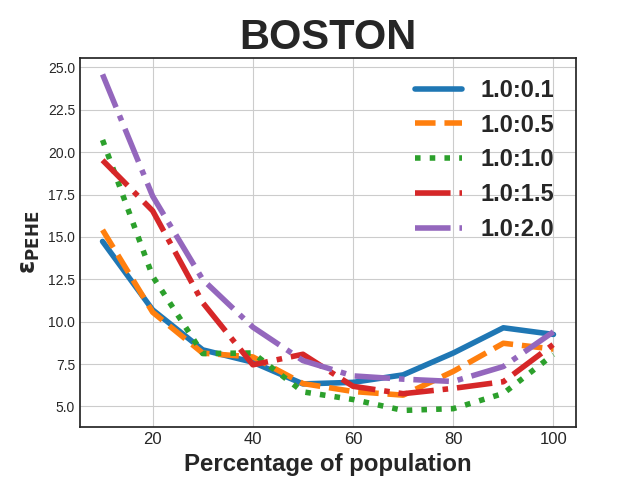}
    \includegraphics[width=0.495\linewidth]{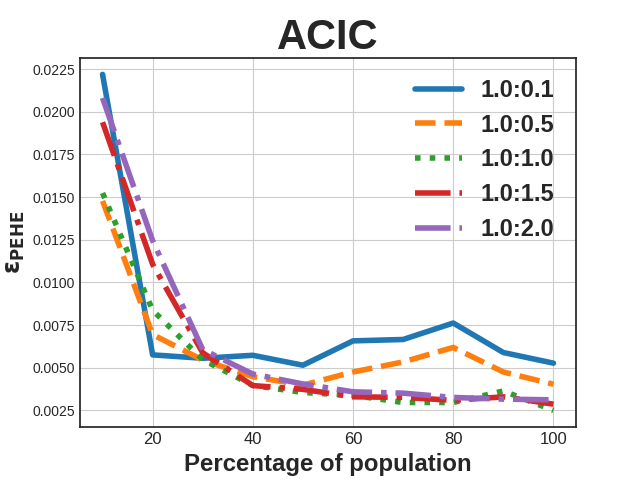}
    \includegraphics[width=0.495\linewidth]{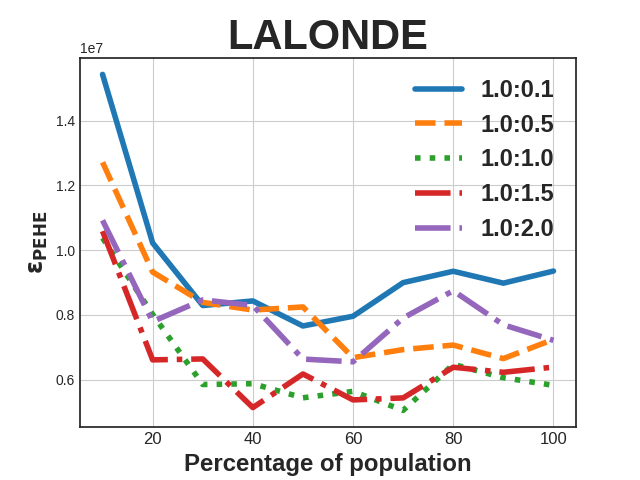}
    \caption{$\epsilon_{PEHE}$ for different $\sigma^0_{\epsilon}:\sigma^1_{\epsilon}$.}
\end{figure}

We also test the hyperparameter sensitivity to $\sigma^a_{\epsilon}$.
The result is in Figure 5.
Except the extreme case ($\sigma^0_{\epsilon}:\sigma^1_{\epsilon}=1.0:0.1$), ABC3 shows robust performance for different $\sigma^a_{\epsilon}$s.
Interestingly, $1.0:0.1$ shows the best performance for IHDP data set, which may imply the noisiness in the control group.
However, $1.0:0.1$ significantly deteriorates on Lalonde data set.
Overall, the equal noise setting ($1.0:1.0$), utilized throughout this paper, is not always the best, but a generally safe choice.

\subsection{Measuring Computation Time}

Here we present the time to sample the whole train data set of Boston data.
We compute the mean and standard deviation of the computation time by iterating 10 times.
As a result, Leverage shows nearly the same computation time as Naive, while ABC3 shows a time comparable to Mackay.
ACE requires nearly twice as much time than ABC3.
However, most policies require less than 1 second to sample the whole data set.
The result shows that ABC3 is a feasible policy with reasonable computation time.

\begin{table}[h]
    \centering
    \begin{tabular}{cccc|c}
        \toprule
         Naive & Leverage & Mackay & ACE & ABC3 \\
         \midrule
         0.2567 & 0.2517 & 0.7595 & 1.6506 & 0.8871 \\
         (0.0321) & (0.0212) & (0.3803) & (0.7961) & (0.0373) \\
         \bottomrule
    \end{tabular}
    \caption{Mean and standard deviation of the computation time (in second) on Boston data set.}
    \label{tab:my_label}
\end{table}

\subsection{Empirical Validation of Assumption}

In Theorem 4.5, we introduced the following assumption: $\epsilon^*(I_n) \leq 2 \delta^*(I_n), \forall I_n$ where $M = \frac{1}{N^2} \Sigma_{i,j=1}^{N} $ $k(x_i,x_j)$, $\delta^*(I_n)=\frac{1}{n} \Sigma_{i \in I_n} \int_{\mathcal{X}} k(x_i, x) d\mathbb{P}(x)$ and $\epsilon^*(I_n) = M -\frac{1}{n^2} \Sigma_{i,j \in I_n} k(x_i, x_j)$.
To verify the empirical satisfaction of the assumption, we compute and plot $\epsilon^*(I_n)$ and $2 \delta^*(I_n)$ for data sets used in our paper.
As computing all $\epsilon^*(I_n)$ and $\delta^*(I_n)$ for all $I_n$ is computationally infeasible, we compute the value of the leading principal submatrix by randomly permuting 100 times and present points with a minimum value of $2 \delta^*(I_n) - \epsilon^*(I_n)$.
The result, shown in Figure 6, supports the empirical satisfaction of the assumption.

\begin{figure}[th]
    \centering
    \includegraphics[width=0.24\linewidth]{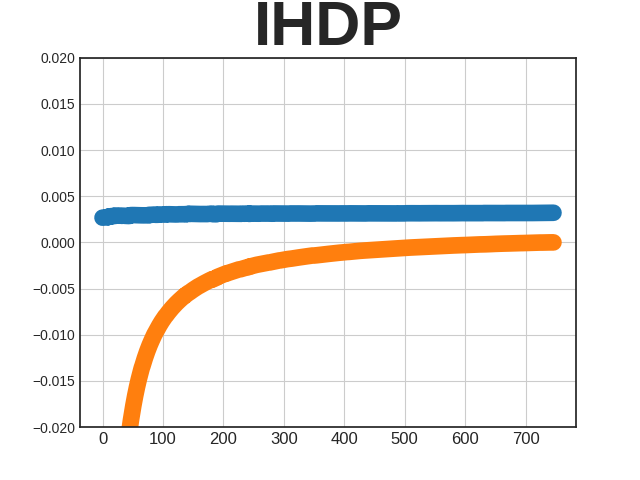}
    \includegraphics[width=0.24\linewidth]{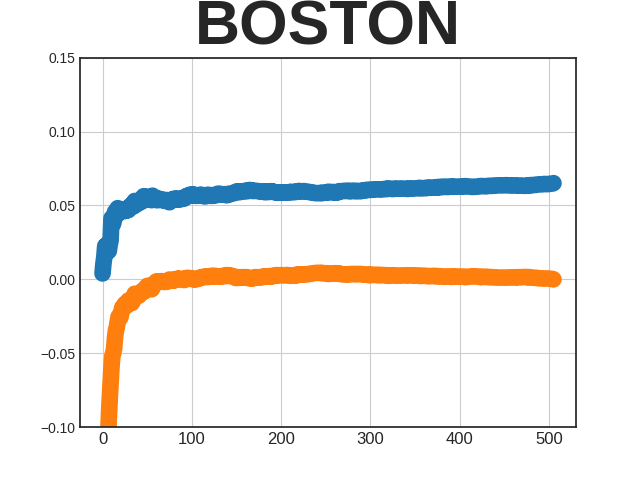}
    \includegraphics[width=0.24\linewidth]{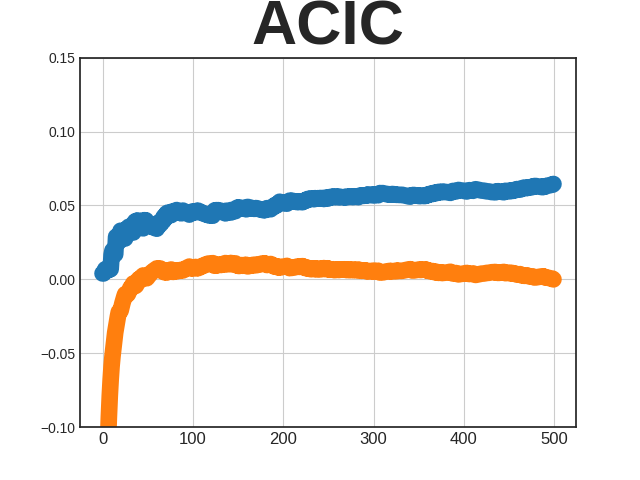}
    \includegraphics[width=0.24\linewidth]{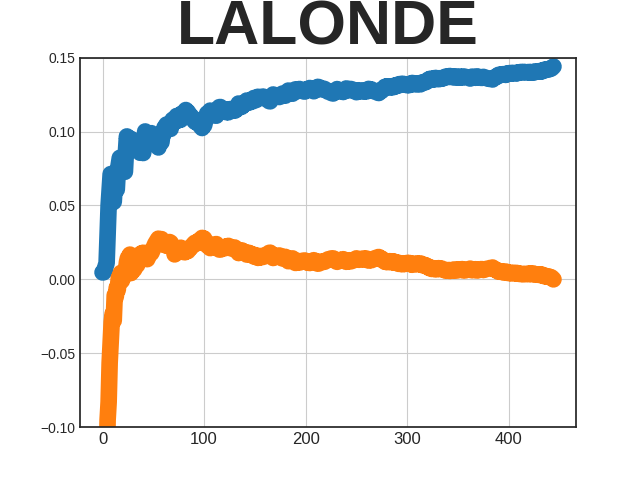}
    \caption{Empirical validation of the assumption ($\epsilon^*(I_n) \leq 2 \delta^*(I_n)$) from Theorem 4.5. The blue line is for $2 \delta^*(I_n)$, the orange line is for $\epsilon^*(I_n)$, and the $x$-axis is for $n$.}
\end{figure}

\section{Discussion \& Limitation}

ABC3 algorithm utilizes Gaussian process at its heart, hence the improvements pertaining to Gaussian process also happen in ABC3. 
For example, as multiple researchers attempted to extend the Gaussian Process to large data sets, e.g. \citet{hensman2013gaussian} and \citet{wang2019exact}, ABC3 can be extended to be more scalable. 
In Appendix E, we showcase a possible direction of extending ABC3 with sample-and-optimize approach, which outperforms the Naive policy on a large Weather data set with much better efficiency.
This demonstrates the potential of ABC3, which goes beyond its original design principle to maximize the learning efficiency of the expensive and limited data set for causal inference. 


Practitioners may consider using ABC3 just as an intelligent sampling policy, in conjunction to external regressors other than Gaussian process. 
The performance of the causal effect estimation when ABC3 is used with a different regressor, like a neural network or a random forest, is reported in Appendix F. 
According to the result, the best choice of regressor may depend not only on the sampling process design but also on the data set itself, which suggests another future research direction on designing optimal regressor given the data set and its sampling algorithm. 

Lastly, ABC3 may be used as a selective data set cleanser that learns to avoid sampling potentially toxic observation points. 
As shown in Figure 1, for Boston and ACIC data sets, ABC3 shows the best performance when observing only a portion of data set. 
The result may imply the existence of \textit{toxic} data points for CATE estimation, and ABC3 successfully avoids sampling those data points unless forced to sample all. 
We believe that more detailed analysis of this behavior deserves a separate future research as well. 

\section{Conclusion}

We present ABC3, an active learning based sampling policy for causal inference.
ABC3 minimizes the expected error of CATE estimation from a Bayesian perspective, without violating the key assumptions in causal inference.
Using maximum mean discrepancy, we prove that ABC3 minimizes the upper bound of imbalance between observed treatment and control groups.
Moreover, ABC3 theoretically minimizes the upper bound of type 1 error probability.
ABC3 empirically outperforms other active learning policies, and its theoretical properties as well as empirical robustness are also validated to give additional support to the general applicability of ABC3. 
We expect ABC3 and its extensions will deepen theoretical insights and general applicability of active learning on casual inference tasks. 

\section*{Acknowledgements}

This work was supported by the Ministry of Trade, Industry and Energy (MOTIE), Korea Institute for Advancement of Technology (KIAT) through the Materials and Parts Industry Technology Development Base Construction Project (P0022334) and National Research Foundation of Korea (NRF) grant funded by the Korea government (MSIT) (No. 2020R1G1A1102828).

\bibliography{aaai25}

\newpage

\end{document}


\begin{appendices}

This is a technical appendix for the manuscript: \textbf{ABC3: Active Bayesian Causal Inference with Cohn Criteria in Randomized Experiments}.

\section{Proofs}
\subsection{Proof on Theorem 4.1.}

To prove the theorem, we need a lemma.

\begin{lemma}
    For a random variable $X$ and two filtration $\mathcal{F}' \subset \mathcal{F}$, if $\mathbb{V} \left[ X | \mathcal{F} \right]$ is a known constant in $\mathcal{F}'$, then $\mathbb{V}\left[ \mathbb{E} \left[ X \right | \mathcal{F}] | \mathcal{F}' \right] = \mathbb{V} \left[ X | \mathcal{F}' \right] - \mathbb{V} \left[ X | \mathcal{F} \right]$
\end{lemma}

\begin{proof}
    \begin{align*}
        \mathbb{V} &\left[ X | \mathcal{F} \right] = \mathbb{E} \left[ (X - \mathbb{E} \left[ X | \mathcal{F} \right])^2 | \mathcal{F} \right] \\
        &= \mathbb{E} \left[ (X - \mathbb{E} \left[ X | \mathcal{F}' \right] + \mathbb{E} \left[ X | \mathcal{F}' \right] - \mathbb{E} \left[ X | \mathcal{F} \right])^2 | \mathcal{F} \right] \\
        &= \mathbb{E} \left[ (X - \mathbb{E} \left[ X | \mathcal{F}' \right])^2 | \mathcal{F} \right] + 2 \mathbb{E} \left[ (X - \mathbb{E} \left[ X | \mathcal{F}' \right]) (\mathbb{E} \left[ X | \mathcal{F}' \right] - \mathbb{E} \left[ X | \mathcal{F} \right]) | \mathcal{F} \right] \\
        & + \mathbb{E} \left[ (\mathbb{E} \left[ X | \mathcal{F}' \right] - \mathbb{E} \left[ X | \mathcal{F} \right])^2 | \mathcal{F} \right] \\
        &= \mathbb{E} \left[ (X - \mathbb{E} \left[ X | \mathcal{F}' \right])^2 | \mathcal{F} \right] + 2 (\mathbb{E} \left[ X | \mathcal{F} \right] - \mathbb{E} \left[ X | \mathcal{F}' \right]) (\mathbb{E} \left[ X | \mathcal{F}' \right] - \mathbb{E} \left[ X | \mathcal{F} \right]) \\
        & + (\mathbb{E} \left[ X | \mathcal{F}' \right] - \mathbb{E} \left[ X | \mathcal{F} \right])^2 \\
        &= \mathbb{E} \left[ (X - \mathbb{E} \left[ X | \mathcal{F}' \right])^2 | \mathcal{F} \right] - (\mathbb{E} \left[ X | \mathcal{F}' \right] - \mathbb{E} \left[ X | \mathcal{F} \right])^2
    \end{align*}
    where the fourth equality holds as $\mathbb{E} \left[ X | \mathcal{F}' \right]$ and $\mathbb{E} \left[ X | \mathcal{F} \right]$ are already known at $\mathcal{F}$.
    By taking an expectation with respect to $\mathcal{F}'$, as $\mathbb{V} \left[ X | \mathcal{F} \right]$ is a known constant in $\mathcal{F}'$,
    \begin{align*}
        \mathbb{V} \left[ X | \mathcal{F} \right] &= \mathbb{E} \left[ (X - \mathbb{E} \left[ X | \mathcal{F}' \right])^2 | \mathcal{F}' \right] - \mathbb{E} \left[ (\mathbb{E} \left[ X | \mathcal{F}' \right] - \mathbb{E} \left[ X | \mathcal{F} \right])^2 | \mathcal{F}' \right] \\
        &= \mathbb{V} \left[ X | \mathcal{F}' \right] - \mathbb{V} \left[ \mathbb{E} \left[ X | \mathcal{F} \right] | \mathcal{F}' \right]
    \end{align*}
    by the tower property.
\end{proof}

\begin{theorem}
    Assume $|k(x, x')| < \infty$ and $|y^a_i| < \infty$ for all $x, x' \in \mathcal{X}$ and $a, i$, as a result, $\epsilon^{\Omega}_{PEHE}(\hat{CATE}_t) < \infty, \forall t$.
    Let our estimator $\hat{CATE}_t(x) = \hat{y}^1_t(x) - \hat{y}^0_t(x)$, where $\hat{y}^a_t(x)=\mathbb{E}_t\left[ Y^a(x) \right]$ is a mean posterior distribution of gaussian process $Y^a$ trained on data set $D_t^a$. 
    Assume two Gaussian processes $Y^1$ and $Y^0$ are independent.
    Then 
    \begin{align*}
     \text{arg min}_{x_{t+1},a_{t+1}} & \mathbb{E}_{t+1} \left[\epsilon_{PEHE}^{\Omega} (\hat{CATE}_{t+1})\right] \\
    & = \text{arg min}_{x_{t+1},a_{t+1}} \int_\mathcal{X} \mathbb{V}_{t+1} \left[ Y^1(x) \right] + \mathbb{V}_{t+1} \left[Y^0(x) \right] d\mathbb{P}(x)
    \end{align*}
\end{theorem}

\begin{proof}
    We can write our target quantity,
    \begin{align*}
        & \mathbb{E}_{t+1} \left[\epsilon_{PEHE}^{\Omega} (\hat{CATE}_{t+1})\right] \\
        & = \mathbb{E}_{t+1} \left[ \int_\mathcal{X} (\hat{CATE}_{t+1}(x) - \hat{CATE}_{\Omega}(x))^2 d\mathbb{P}(x) \right]
    \end{align*}
    Note
    \begin{align*}
        & \mathbb{E}_{t+1} \left[ \int_\mathcal{X} (\hat{CATE}_{t+1}(x) - \hat{CATE}_{\Omega}(x))^2 d\mathbb{P}(x)  \right] \\
        &= \int_\mathcal{X} \mathbb{E}_{t+1} \left[ (\hat{CATE}_{t+1}(x) - \hat{CATE}_{\Omega}(x))^2 \right] d\mathbb{P}(x) \text{ by Fubini}\\
        &= \int_\mathcal{X} \mathbb{E}_{t+1} \left[ \hat{CATE}_{t+1}(x) - \hat{CATE}_{\Omega}(x) \right]^2 + \mathbb{V}_{t+1} \left[ \hat{CATE}_{t+1}(x) - \hat{CATE}_{\Omega}(x) \right] d\mathbb{P}(x) \\
        &= \int_\mathcal{X} \mathbb{V}_{t+1} \left[ \hat{CATE}_{t+1}(x) - \hat{CATE}_{\Omega}(x) \right] d\mathbb{P}(x) \\
        &= \int_\mathcal{X} \mathbb{V}_{t+1} \left[ \hat{CATE}_{\Omega}(x) \right] d\mathbb{P}(x)
    \end{align*}
    The fourth line holds because $\mathbb{E}_{t+1} \left[ \hat{CATE}_{\Omega}(x) \right] = \mathbb{E}_{t+1} \left[ \hat{CATE}_{t+1}(x) \right]$ by the tower property.
    We can use Fubini Theorem in the second line as we assumed $\epsilon^{\Omega}_{PEHE}(\hat{CATE}_t) < \infty, \forall t$.
    Also, $\mathbb{V}_{t+1} \left[ \hat{CATE}_{t+1}(x) \right] = 0$ as $\hat{CATE}_{t+1}$ is fixed at $t+1$.
    
    By the property of the Gaussian process, we can compute $\mathbb{V}_{\Omega} \left[ Y^1(x) - Y^0(x) \right]$, i.e. it is a known constant at $t+1$. Then by applying Lemma A.1., 
    \begin{align*}
        \int_\mathcal{X} \mathbb{V}_{t+1} & \left[ \hat{CATE}_{\Omega}(x) \right] d\mathbb{P}(x) \\
        &= \int_\mathcal{X} \mathbb{V}_{t+1} \left[ \mathbb{E}_{\Omega} \left[ Y^1(x) - Y^0(x) \right] \right] d\mathbb{P}(x) \\
        &= \int_\mathcal{X} \mathbb{V}_{t+1} \left[ Y^1(x) - Y^0(x) \right] - \mathbb{V}_{\Omega} \left[ Y^1(x) - Y^0(x) \right] d\mathbb{P}(x)
    \end{align*}
    As our choice on $x_{t+1}$ and $a_{t+1}$ does not affect $\mathbb{V}_{\Omega} \left[ Y^1(x) - Y^0(x) \right]$, we obtain
    \begin{align*}
     \text{arg min}_{x_{t+1},a_{t+1}} & \mathbb{E}_{t+1} \left[\epsilon_{PEHE}^{\Omega} (\hat{CATE}_{t+1})\right] \\
    & = \text{arg min}_{x_{t+1},a_{t+1}} \int_\mathcal{X} \mathbb{V}_{t+1} \left[ Y^1(x) - Y^0(x) \right] d\mathbb{P}(x) \\
    & = \text{arg min}_{x_{t+1},a_{t+1}} \int_\mathcal{X} \mathbb{V}_{t+1} \left[ Y^1(x) \right] + \mathbb{V}_{t+1} \left[Y^0(x) \right] d\mathbb{P}(x)
    \end{align*}
    by the independence assumption between $Y^1$ and $Y^0$.
\end{proof}

\subsection{Proof on Proposition 4.2.}

To prove the proposition, we need a lemma.

\begin{lemma}
    \begin{align*}
        & \text{arg min}_{x_{t+1},a} \int_\mathcal{X} \mathbb{V}_{t+1} \left[ Y^1(x) \right] + \mathbb{V}_{t+1} \left[ Y^0(x) \right] d\mathbb{P}(x) \\
        & = \text{arg max}_{x_{t+1},a} \int_\mathcal{X} \mathbb{V}_{t} \left[ Y^a(x) \right] - \mathbb{V}_{t+1} \left[ Y^a(x) \right]  d\mathbb{P}(x)
    \end{align*}
\end{lemma}

\begin{proof}
    $x_{t+1}=\tilde{x}, a_{t+1}=a$ holds, if
    \begin{align}
        & = \int_\mathcal{X} \mathbb{V}_{\tilde{t+1}} \left[ Y^a(x) \right] + \mathbb{V}_{t} \left[ Y^{|1-a|}(x) \right] d\mathbb{P}(x) \\
        & \leq \int_\mathcal{X} \mathbb{V}_{\bar{t+1}} \left[ Y^{|1-a|}(x) \right] + \mathbb{V}_{t} \left[ Y^a(x) \right] d\mathbb{P}(x)
    \end{align}
    and
    \begin{align}
        & = \int_\mathcal{X} \mathbb{V}_{\tilde{t+1}} \left[ Y^a(x) \right] + \mathbb{V}_{t} \left[ Y^{|1-a|}(x) \right] d\mathbb{P}(x) \\
        & \leq \int_\mathcal{X} \mathbb{V}_{\bar{t+1}} \left[ Y^a(x) \right] + \mathbb{V}_{t} \left[ Y^{|1-a|}(x) \right] d\mathbb{P}(x)
    \end{align}
    for all $\bar{x}$, where $\mathbb{V}_{\tilde{t+1}} = \mathbb{V} \left[ \cdot | \mathcal{F}_{\tilde{t+1}} \right], \mathbb{V}_{\bar{t+1}} = \mathbb{V} \left[ \cdot | \mathcal{F}_{\bar{t+1}} \right], \mathcal{F}_{\tilde{t+1}} = \mathcal{F}_t \bigcup \{\tilde{x}\}, \mathcal{F}_{\bar{t+1}} = \mathcal{F}_t \bigcup \{\bar{x}\}$, i.e. filtration assuming the future observation $x_{t+1}=\tilde{x}$ and $\bar{x}$. 
    Proposition holds as
    \begin{align*}
        \text{(2)} \iff & \int_\mathcal{X} \mathbb{V}_{t} \left[ Y^a(x) \right] - \mathbb{V}_{\tilde{t+1}} \left[ Y^a(x) \right] d\mathbb{P}(x) \\
         \geq & \int_\mathcal{X} \mathbb{V}_{t} \left[ Y^{|1-a|}(x) \right] - \mathbb{V}_{\bar{t+1}} \left[ Y^{|1-a|}(x) \right] d\mathbb{P}(x)
    \end{align*}
    and
    \begin{align*}
        \text{(4)} \iff & \int_\mathcal{X} \mathbb{V}_{\tilde{t+1}} \left[ Y^a(x) \right] d\mathbb{P}(x) \leq \int_\mathcal{X}  \mathbb{V}_{\bar{t+1}} \left[ Y^a(x) \right] d\mathbb{P}(x) \\
        \iff & \int_\mathcal{X} - \mathbb{V}_{\tilde{t+1}} \left[ Y^a(x) \right] d\mathbb{P}(x) \geq \int_\mathcal{X}  - \mathbb{V}_{\bar{t+1}} \left[ Y^a(x) \right] d\mathbb{P}(x) \\
        \iff & \int_\mathcal{X} \mathbb{V}_{t} \left[ Y^a(x) \right] - \mathbb{V}_{\tilde{t+1}} \left[ Y^a(x) \right] d\mathbb{P}(x) \\
        & \geq \int_\mathcal{X} \mathbb{V}_{t} \left[ Y^a(x) \right] - \mathbb{V}_{\bar{t+1}} \left[ Y^a(x) \right] d\mathbb{P}(x)
    \end{align*}
\end{proof}

Now we can prove the proposition

\begin{proposition}
    \begin{align*}
         & \text{arg min}_{x_{t+1},a_{t+1}} \int_\mathcal{X} \mathbb{V}_{t+1} \left[ Y^1(x) \right] + \mathbb{V}_{t+1} \left[ Y^0(x) \right] d\mathbb{P}(x) = \\
        & \text{arg max}_{x_{t+1},a} \frac{ \int_\mathcal{X} \left[ (\mathbf{\tilde{k}}^a_{t+1})^T \left[ \mathbb{K}^a_t + \sigma^2_\epsilon \mathbf{I} \right]^{-1} \mathbf{k}^a_{t,*}(x) - k(x_{t+1}, x)\right]^2 d\mathbb{P}(x) }{k(x_{t+1},x_{t+1}) + \sigma^2_\epsilon - (\mathbf{\tilde{k}}^a_{t+1})^T  \left[ \mathbb{K}^a_t + \sigma^2_\epsilon \mathbf{I} \right]^{-1} \mathbf{\tilde{k}}^a_{t+1}} \\
    \end{align*}
    where $\mathbf{\tilde{k}}^a_{t+1} = \left[k(x_i, x_{t+1})\right]_{i \in I^a_t}$
\end{proposition}

\begin{proof}
    As shown in Lemma A.3., $\text{arg min}_{x_{t+1},a_{t+1}} \int_\mathcal{X} \mathbb{V}_{t+1} \left[ Y^1(x) \right] + \mathbb{V}_{t+1} \left[ Y^0(x) \right] d\mathbb{P}(x) = \text{arg max}_{x_{t+1},a} \mathbb{V}_{t} \left[ Y^a(x) \right] - \mathbb{V}_{t+1} \left[ Y^a(x) \right]$.
    
    Note,
    \begin{align*}
        & \mathbb{V}_{t} \left[ Y^a(x) \right] - \mathbb{V}_{t+1} \left[ Y^a(x) \right] \\ 
        & =  (\mathbf{k}^a_{t+1, *}(x))^T \left[ \mathbb{K}^a_{t+1} + \sigma^2_\epsilon \mathbf{I} \right]^{-1} \mathbf{k}^a_{t+1,*}(x) - (\mathbf{k}^a_{t, *}(x))^T \left[ \mathbb{K}^a_t + \sigma^2_\epsilon \mathbf{I} \right]^{-1} \mathbf{k}^a_{t,*}(x)
    \end{align*}
    and
    \begin{align*}
        \mathbb{K}^a_{t+1} &= 
        \begin{bmatrix}
        \mathbb{K}^a_{t} & \mathbf{\tilde{k}}^a_{t+1} \\
        (\mathbf{\tilde{k}}^a_{t+1})^T & k(x_{t+1}, x_{t+1})
        \end{bmatrix} \\
        \mathbf{k}^a_{t+1, *}(x) &= 
        \begin{bmatrix}
        \mathbf{k}^a_{t, *}(x) & k(x_{t+1}, x)
        \end{bmatrix} \\
    \end{align*}
    By utilizing the bordering method \citep{faddeev1981lin}, we can compute the inverse matrix.
    \begin{equation*}
        (\mathbb{K}^a_{t+1} + \sigma^2_{\epsilon} \mathbf{I})^{-1} =
        \begin{bmatrix}
            (i) & - \frac{(\mathbb{K}^a_t + \sigma^2_{\epsilon} \mathbf{I})^{-1} \tilde{\mathbf{k}}^a_{t+1}}{(ii)} \\
            & \\  
            - \frac{(\tilde{\mathbf{k}}^a_{t+1})^T (\mathbb{K}^a_t + \sigma^2_{\epsilon} \mathbf{I})^{-1}}{(ii)} &  \frac{1}{(ii)} \\
            \label{eq:bordering}
        \end{bmatrix}
    \end{equation*}
    where
    \begin{align*}
        (i) &= (\mathbb{K}^a_t + \sigma^2_{\epsilon} \mathbf{I})^{-1} + \frac{(\mathbb{K}^a_t + \sigma^2_{\epsilon} \mathbf{I})^{-1} \tilde{\mathbf{k}}^a_{t+1} (\tilde{\mathbf{k}}^a_{t+1})^T (\mathbb{K}^a_t + \sigma^2_{\epsilon} \mathbf{I})^{-1}}{(ii)} \\
        (ii) &= k(x_{t+1}, x_{t+1}) + \sigma^2_{\epsilon} - (\tilde{\mathbf{k}}^a_{t+1})^T (\mathbb{K}^a_t + \sigma^2_{\epsilon} \mathbf{I})^{-1} \tilde{\mathbf{k}}^a_{t+1}
    \end{align*}
    By multiplying $\mathbf{k}^a_{t+1, *}$ to both sides, we obtain
    \begin{align*}
        (\mathbf{k}^a_{t+1, *}(x))^T & \left[ \mathbb{K}^a_{t+1} + \sigma^2_\epsilon \mathbf{I} \right]^{-1} \mathbf{k}^a_{t+1,*}(x) \\
        &= \frac{ \left[ (\mathbf{\tilde{k}}^a_{t+1})^T \left[ \mathbb{K}^a_t + \sigma^2_\epsilon \mathbf{I} \right]^{-1} \mathbf{k}^a_{t,*}(x) - k(x_{t+1}, x)\right]^2 }{k(x_{t+1},x_{t+1}) + \sigma^2_\epsilon - (\mathbf{\tilde{k}}^a_{t+1})^T  \left[ \mathbb{K}^a_t + \sigma^2_\epsilon \mathbf{I} \right]^{-1} \mathbf{\tilde{k}}^a_{t+1}} \\
        & + (\mathbf{k}^a_{t, *}(x))^T \left[ \mathbb{K}^a_t + \sigma^2_\epsilon \mathbf{I} \right]^{-1} \mathbf{k}^a_{t,*}(x)
    \end{align*}
    So $\mathbb{V}_{t} \left[ Y^a(x) \right] - \mathbb{V}_{t+1} \left[ Y^a(x) \right] = \frac{ \left[ (\mathbf{\tilde{k}}^a_{t+1})^T \left[ \mathbb{K}^a_t + \sigma^2_\epsilon \mathbf{I} \right]^{-1} \mathbf{k}^a_{t,*}(x) - k(x_{t+1}, x)\right]^2 }{k(x_{t+1},x_{t+1}) + \sigma^2_\epsilon - (\mathbf{\tilde{k}}^a_{t+1})^T  \left[ \mathbb{K}^a_t + \sigma^2_\epsilon \mathbf{I} \right]^{-1} \mathbf{\tilde{k}}^a_{t+1}}$
    so the proposition holds.
\end{proof}

\subsection{Proof on Theorem 4.5}

\begin{theorem} 
    Assume $0 < |k(x, x')| < \infty$ and $|y^a_i| < \infty$ for all $x, x' \in \mathcal{X}$ and $a, i$. 
    Let $\lambda^*$ be a maximum eigenvalue of $\mathbb{K}_{\Omega}$, i.e. covariance matrix of whole covariates.
    Let $M = \frac{1}{N^2} \Sigma_{i,j=1}^{N} k(x_i,x_j)$.
    Define functions $\delta^*(I_n)=\frac{1}{n} \Sigma_{i \in I_n} \int_{\mathcal{X}} k(x_i, x) d\mathbb{P}(x)$, and $\epsilon^*(I_n) = M - \frac{1}{n^2} \Sigma_{i,j \in I_n} k(x_i, x_j)$, where $I_n \subset \{1,...,N\}$ is an n-element index set.
    Assume  $\epsilon^*(I_n) \leq 2 \delta^*(I_n), \forall I_n$.
    Then
    \begin{align*}
        & MMD(\mathbb{P}^1_t, \mathbb{P}^0_t, \mathcal{F})^2 \leq  4 \frac{\lambda^*}{|I^1_t|} + 4 \frac{\lambda^*}{|I^0_t|} + 2 \int_\mathcal{X} \mathbb{V}_t \left[ Y^1(x) \right] + \mathbb{V}_t \left[ Y^0(x) \right] d\mathbb{P}(x)
    \end{align*}
\end{theorem}

\begin{proof}
    Note, $MMD(\mathbb{P}^1_t, \mathbb{P}^0_t, \mathcal{F})^2 \leq 2 MMD(\mathbb{P}^1_t, \mathbb{P}, \mathcal{F})^2 + 2MMD(\mathbb{P}^0_t, \mathbb{P}, \mathcal{F})^2$, as
    \begin{align*}
        & MMD(\mathbb{P}^1_t, \mathbb{P}^0_t, \mathcal{F})^2 \\
        & \leq \left[ MMD(\mathbb{P}^1_t, \mathbb{P}, \mathcal{F}) + MMD(\mathbb{P}^0_t, \mathbb{P}, \mathcal{F}) \right]^2 \text{ by Triangle Inequality} \\
        & =  MMD(\mathbb{P}^1_t, \mathbb{P}, \mathcal{F})^2 + MMD(\mathbb{P}^0_t, \mathbb{P}, \mathcal{F})^2 + 2 MMD(\mathbb{P}^1_t, \mathbb{P}, \mathcal{F}) MMD(\mathbb{P}^0_t, \mathbb{P}, \mathcal{F}) \\
        & \leq 2 MMD(\mathbb{P}^1_t, \mathbb{P}, \mathcal{F})^2 + 2MMD(\mathbb{P}^0_t, \mathbb{P}, \mathcal{F})^2
    \end{align*}
    We focus on one $MMD(\mathbb{P}^a_t, \mathbb{P}, \mathcal{F})$ as we can extend the arguments to $MMD(\mathbb{P}^{|1-a|}_t, \mathbb{P}, \mathcal{F})$.
    For notational convenience, we omit $a$ if trivial.
    
    Define a weighted empirical distribution $\tilde{\mathbb{P}_t}=\Sigma_{i \in I_t} \omega_{t,i} \delta_{x_i}$, where the weight vector $\omega_t = (\mathbb{K}_t)^{-1} \mathbf{\tilde{k}}_t$, $\mathbf{\tilde{k}}_t = \left[ \int k(x_i, x) d\mathbb{P}(x) \right]_{i \in I_t}$.
    As
    \begin{align}
        MMD(\mathbb{P}_t, \mathbb{P}, \mathcal{F})^2 &= ||\mu_{\mathbb{P}_t} - \mu_{\mathbb{P}}||^2 \\
        &= \bar{\mathbf{1}}^T \mathbb{K}_t \bar{\mathbf{1}} + \int_\mathcal{X} \int_\mathcal{X} k(x, x') d\mathbb{P}(x) d\mathbb{P}(x') - 2 \mathbf{\tilde{k}}_t^T \bar{\mathbf{1}} \\
        &= \bar{\mathbf{1}}^T \mathbb{K}_t \bar{\mathbf{1}} \\
        &+ \int_\mathcal{X} \int_\mathcal{X} k(x, x') d\mathbb{P}(x) d\mathbb{P}(x') - \mathbf{\tilde{k}}_t^T (\mathbb{K}_t)^{-1} \mathbf{\tilde{k}}_t \\
        &+ \mathbf{\tilde{k}}_t^T (\mathbb{K}_t)^{-1} \mathbf{\tilde{k}}_t - 2 \mathbf{\tilde{k}}_t^T \bar{\mathbf{1}}
    \end{align}
    where $\bar{\mathbf{1}} = \frac{1}{|I_t|} \mathbf{1}_{|I_t|}$ is a vector computing the mean.
    As $\mathbb{K}_t$ is a symmetric and positive-definite matrix, we obtain,
    \begin{equation}
        \text{(7)} \leq \lambda_{max}(\mathbb{K}_t) \bar{\mathbf{1}}^T \bar{\mathbf{1}} = \frac{1}{|I_t|} \lambda_{max}(\mathbb{K}_t) \leq \frac{1}{|I_t|} \lambda_{max}(\mathbb{K}_{\Omega}) = \frac{\lambda^*}{|I_t|}
    \end{equation}
    where $\lambda_{max}$ is a maximum eigenvalue of a matrix.
    
    From the Bayesian quadrature literature, it is known that $\text{(8)} = \mathbb{V}_t \left[ \int Y(x) d\mathbb{P}(x) \right]$
    (See Section 6.2. of \citet{kanagawa2018gaussian}).
    Since
    \begin{align}
        \mathbb{V}_t \left[\int Y(x) d\mathbb{P}(x) \right] &= \mathbb{E}_t \left[ \left[ \int Y(x) d\mathbb{P}(x) - \mathbb{E}_t \left[ \int Y(x) d\mathbb{P}(x) \right] \right]^2 \right] \\
        &= \mathbb{E}_t \left[ \left[ \int Y(x) - \mathbb{E}_t \left[ Y(x) \right] d\mathbb{P}(x) \right]^2 \right] \text{ by Fubini} \\
        &\leq \mathbb{E}_t \left[ \int \left[  Y(x) - \mathbb{E}_t \left[ Y(x) \right] \right]^2 d\mathbb{P}(x) \right] \text{ by Jensen's inequality} \\
        &= \int \mathbb{E}_t \left[ \left[  Y(x) - \mathbb{E}_t \left[ Y(x) \right] \right]^2 \right] d\mathbb{P}(x) \text{ by Fubini} \\
        &= \int \mathbb{V}_t \left[ Y(x) \right] d\mathbb{P}(x)
    \end{align}
    We can apply the Fubini theorem to (12) because,
    \begin{align*}
        \int \mathbb{E}_t \left[ Y(x) \right] d\mathbb{P}(x) & = \int \mathbf{k}_{t,*}(x)^T (\mathbb{K}_t)^{-1} \mathbf{y}_t d\mathbb{P}(x) \\
        & = \frac{1}{N} \Sigma_{i=1}^{N} \mathbf{k}_{t,*}(x_i)^T (\mathbb{K}_t)^{-1} \mathbf{y}_t < \infty
    \end{align*}
    by assumptions. 
    Likewise for (14), 
    \begin{align*}
        \int \mathbb{E}_t \left[ \left[  Y(x) - \mathbb{E}_t \left[ Y(x) \right] \right]^2 \right] & d\mathbb{P}(x) = \int \mathbb{V}_t \left[ Y(x) \right] d\mathbb{P}(x) \\
        & = \int k(x, x) - \mathbf{k}_{t,*}(x)^T (\mathbb{K}_t)^{-1} \mathbf{k}_{t,*}(x) d\mathbb{P}(x) \\
        & = \frac{1}{N} \Sigma_{i=1}^{N} k(x_i, x_i) - \mathbf{k}_{t,*}(x_i)^T (\mathbb{K}_t)^{-1} \mathbf{k}_{t,*}(x_i) < \infty
    \end{align*}
    So
    \begin{equation}
        \text{(8)} \leq \int \mathbb{V}_t \left[ Y(x) \right] d\mathbb{P}(x)
    \end{equation}
    As $MMD(\mathbb{P}, \tilde{\mathbb{P}}_t, \mathcal{F})^2 = \int_\mathcal{X} \int_\mathcal{X} k(x, x') d\mathbb{P}(x) d\mathbb{P}(x') - \mathbf{\tilde{k}}_t^T (\mathbb{K}_t)^{-1} \mathbf{\tilde{k}}_t \geq 0$, we have $M = \frac{1}{N^2} \Sigma_{i,j=1}^{N} k(x_i,x_j) = \int_\mathcal{X} \int_\mathcal{X} k(x, x') d\mathbb{P}(x) d\mathbb{P}(x') \geq \mathbf{\tilde{k}}_t^T (\mathbb{K}_t)^{-1} \mathbf{\tilde{k}}_t$.
    By definition, $2 \mathbf{\tilde{k}}_t^T \bar{\mathbf{1}} = 2 \delta^*(I_t)$ and $\bar{\mathbf{1}}^T \mathbb{K}_t \bar{\mathbf{1}} = M - \epsilon^*(I_t)$.
    By our assumption,
    \begin{align}
        \epsilon^*(I_t) \leq 2 \delta^*(I_t) & \iff M - \epsilon^*(I_t) \geq M - 2 \delta^*(I_t) \\
        & \iff \bar{\mathbf{1}}^T \mathbb{K}_t \bar{\mathbf{1}} \geq M - 2 \mathbf{\tilde{k}}_t^T \bar{\mathbf{1}} \\
        & \implies \bar{\mathbf{1}}^T \mathbb{K}_t \bar{\mathbf{1}} \geq \mathbf{\tilde{k}}_t^T (\mathbb{K}_t)^{-1} \mathbf{\tilde{k}}_t - 2 \mathbf{\tilde{k}}_t^T \bar{\mathbf{1}} \\
        & \implies \frac{\lambda^*}{|I_t|} \geq \mathbf{\tilde{k}}_t^T (\mathbb{K}_t)^{-1} \mathbf{\tilde{k}}_t - 2 \mathbf{\tilde{k}}_t^T \bar{\mathbf{1}} \text{ by (10)}
    \end{align}
    By combining the results from (10), (16), and (20), we obtain
    \begin{equation*}
        MMD(\mathbb{P}^a_t, \mathbb{P}, \mathcal{F})^2 \leq 2 \frac{\lambda^*}{|I^a_t|} + \int \mathbb{V}_t \left[ Y^a(x) \right] d\mathbb{P}(x)
    \end{equation*}
    and as a result,
    \begin{equation*}
        MMD(\mathbb{P}^1_t, \mathbb{P}^0_t, \mathcal{F})^2 \leq 4 \frac{\lambda^*}{|I^1_t|} + 4 \frac{\lambda^*}{|I^0_t|} + 2 \int \mathbb{V}_t \left[ Y^1(x) \right] + \mathbb{V}_t \left[ Y^0(x) \right] d\mathbb{P}(x)
    \end{equation*}
\end{proof}

\subsection{Proof on Theorem 4.7.}

\begin{theorem}
    Under Fisher's Sharp null hypothesis, $\text{arg min}_{x_{t+1}, a_{t+1}} \int_\mathcal{X} \mathbb{V}_{t+1} \left[ Y^1(x) \right] + \mathbb{V}_{t+1} \left[ Y^0(x) \right] d\mathbb{P}(x)$ also minimizes the upper bound of the $\int_{\mathcal{X}} \mathbb{P}_{t+1} \left[ \text{Type 1 Error}(x) \right] d\mathbb{P}(x)$
\end{theorem}

\begin{proof}
    Note,
    \begin{equation*}
        \mathbb{P}_{t+1}\left[ \text{Type 1 Error}(x) \right] =  \mathbb{P}_{t+1}\left[ |Y^1(x) - Y^0(x)| > \alpha \right]\\
    \end{equation*}
    By the Markov inequality,
    \begin{align*}
        & \int_\mathcal{X} \mathbb{P}_{t+1}\left[ |Y^1(x) - Y^0(x)| > \alpha \right] d\mathbb{P}(x) \\
        & = \int_\mathcal{X} \mathbb{P}_{t+1}\left[ (Y^1(x) - Y^0(x))^2 > \alpha^2 \right] d\mathbb{P}(x) \\
        & \leq \int_\mathcal{X} \frac{\mathbb{E}_{t+1} \left[ (Y^1(x) - Y^0(x))^2 \right]}{\alpha^2} d\mathbb{P}(x) \\
        & = \int_\mathcal{X} \frac{\mathbb{V}_{t+1} \left[ (Y^1(x) - Y^0(x)) \right] + \mathbb{E}_{t+1} \left[ (Y^1(x) - Y^0(x)) \right]^2}{\alpha^2} d\mathbb{P}(x) \\
        & \propto \int_\mathcal{X} \mathbb{V}_{t+1} \left[ Y^1(x) \right] + \left[ Y^0(x) \right] + \hat{CATE}_{t+1}(x)^2 d\mathbb{P}(x) \\
        & = \int_\mathcal{X} \mathbb{V}_{t+1} \left[ Y^1(x) \right] + \mathbb{V}_{t+1} \left[ Y^0(x) \right] + ( \hat{CATE}_{t+1}(x) - \hat{CATE}_{\Omega}(x) )^2 d\mathbb{P}(x) \\
        & = \int_\mathcal{X} \left[ \mathbb{V}_{t+1} \left[ Y^1(x) \right] + \mathbb{V}_{t+1} \left[ Y^0(x) \right] \right] + \mathbb{E}_{t+1} \left[ ( \hat{CATE}_{t+1}(x) - \hat{CATE}_{\Omega}(x) )^2 \right] d\mathbb{P}(x) \\
        & = \int_\mathcal{X} \left[ \mathbb{V}_{t+1} \left[ Y^1(x) \right] + \mathbb{V}_{t+1} \left[ Y^0(x) \right] \right] d\mathbb{P}(x) + \mathbb{E}_{t+1} \left[ \epsilon^{\Omega}_{PEHE}(\hat{CATE}_{t+1}) \right]
    \end{align*}
    where the last three lines hold as $\hat{CATE}_{\Omega}(x)=0, \forall x$ under the null hypothesis, and $\hat{CATE}_{t+1}(x) - \hat{CATE}_{\Omega}(x)$ is known at $t+1$.
    The last line holds by applying the Fubini Theorem with our assumption that $\epsilon^{\Omega}_{PEHE}(\hat{CATE}_t) < \infty, \forall t$.
    As shown in Theorem 4.1., both terms are our minimization target.
\end{proof}

\section{Analysis on the Assumption of Theorem 4.5.}
    
Recall the key assumption in Theorem 4.5. was $\epsilon^*(I_n) \leq 2 \delta^*(I_n), \forall I_n$ where $M = \frac{1}{N^2} \Sigma_{i,j=1}^{N} $ $k(x_i,x_j)$, $\delta^*(I_n)=\frac{1}{n} \Sigma_{i \in I_n} \int_{\mathcal{X}} k(x_i, x) d\mathbb{P}(x)$ and $\epsilon^*(I_n) = M -\frac{1}{n^2}$\\
$ \Sigma_{i,j \in I_n} k(x_i, x_j)$.
The assumption is equivalent to $M - 2 \delta^*(I_n) \leq \frac{1}{n^2} \Sigma_{i,j \in I_n} k(x_i, x_j)$.
Assume the kernel is a non-negative valued kernel with maximum value 1, which is achieved when $x=x'$ (e.g. RBF and Matern kernels, used in our paper).

Then $\text{inf}_{I_n} \frac{1}{n^2} \Sigma_{i,j \in I_n} k(x_i, x_j)$ is lower bounded by $\frac{1}{n}$.
The lower bound is achieved when the covariates are (roughly speaking) extremely and absolutely isolated in the kernel space, i.e. $k(x_i,x_j) = 0, \forall i \neq j$.
The gap between the line and $\text{inf}_{I_n} \frac{1}{n^2} \Sigma_{i,j \in I_n} k(x_i, x_j)$ represents the absolute uniformness of the covariates.
Meanwhile, $M - 2 \text{inf}_{I_n} \delta^*(I_n)$ is lower bounded by the straight line connecting $(0, M)$ and $(N, -M)$.
The lower bound is achieved when the covariates are (again, roughly speaking) extremely uniformly distributed, i.e. $\int k(x_i,x) d\mathbb{P}(x) = \int k(x_j,x) d\mathbb{P}(x), \forall i, j$.
The gap between the line and $M - 2 \text{inf}_{I_n} \delta^*(I_n)$ represents the relative isolatedness of the covariates.

So the assumption may not hold if the covariates are absolutely and relatively isolated.
However, as shown in Section 5.5, this assumption holds for general data sets with a large margin.
Figure 1 visualizes the condition.

\begin{figure*}[h]
    \centering
    \includegraphics[width=0.8\linewidth]{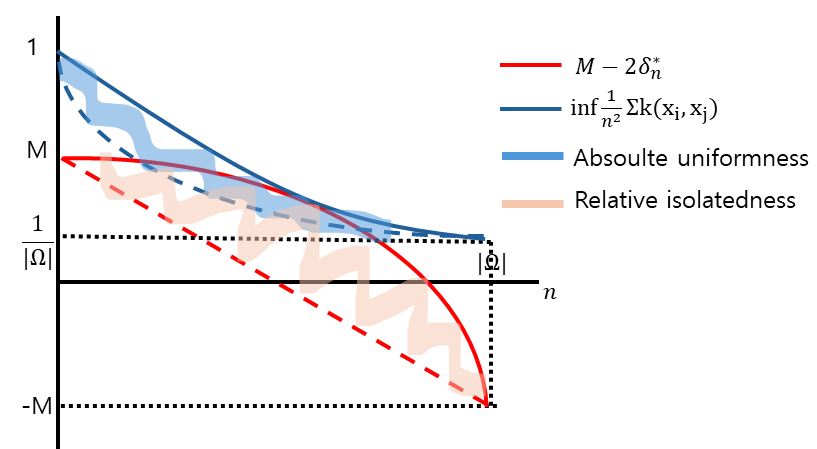}
    \caption{Intuitive visualization of the assumption used in Theorem 4.5.}
\end{figure*}



\section{Data Set \& Experimental Setting}

We present more detailed explanations of utilized data sets.

\textbf{IHDP}: (\citet{brooks1992effects} and \citet{hill2011bayesian}): This study measures the cognitive development of low-birth-weight and premature infants after intensive high-quality child care and home visits. It consists of 747 samples with 25 covariates describing the children and parents with simulated potential outcomes.

\textbf{Boston} \citep{harrison197881hedonic}: This data set measures the median house prices in the Boston area. It consists of 506 samples with 13 covariates. Following \citet{addanki2022sample}, we set $Y^0=Y^1$ to simulate the null hypothesis circumstance (i.e. $CATE(x)=0, \forall x$).

\textbf{ACIC} \citep{gruber2019acic}: We utilize the first test data set of the Atlantic Causal Inference Conference 2019 Data Challenge data set. It consists of 500 samples with 22 covariates. We report the results for the other 7 test data sets on Appendix G.

\textbf{Lalonde} \citep{lalonde1986evaluating}: This data set measures the effectiveness of a job training program on each individual's earnings. It consists of 506 samples with 10 covariates. As in Boston, we simulate the null hypothesis circumstance.

For the experiment, we use a machine with AMD Ryzen 9 5900X 12-Core Processor CPU and NVIDIA RTX 3090 GPU.

\section{Standard Deviation of Measured $\epsilon_{PEHE}$}

We present the standard deviation of measured $\epsilon_{PEHE}$.

\begin{figure*}[h]
    \centering
    \includegraphics[width=0.4\linewidth]{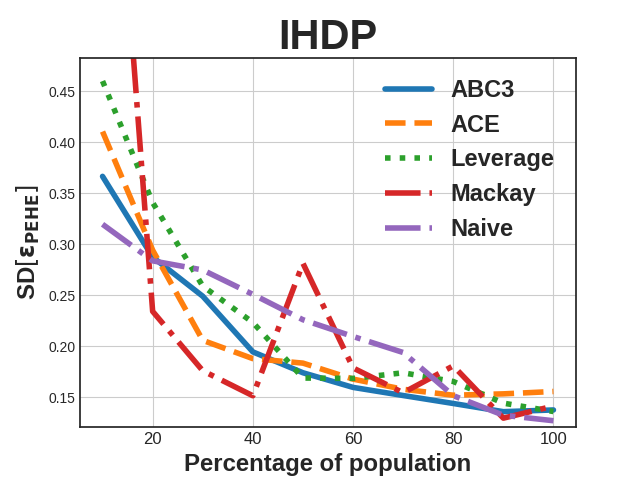}
    \includegraphics[width=0.4\linewidth]{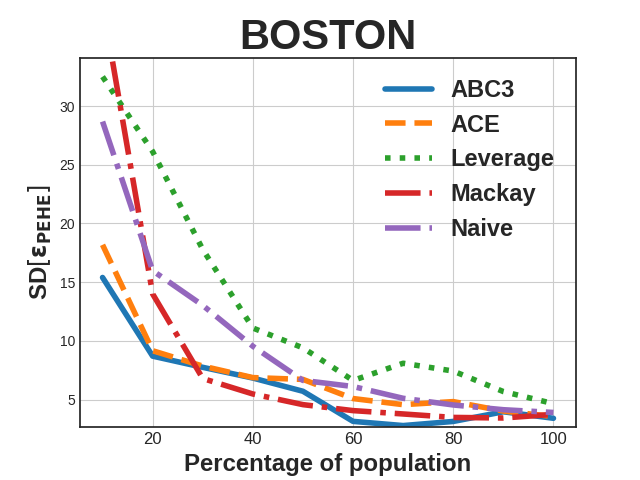}
    \includegraphics[width=0.4\linewidth]{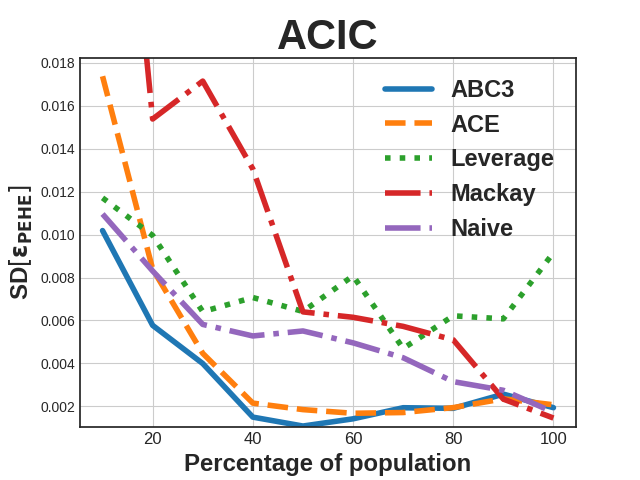}
    \includegraphics[width=0.4\linewidth]{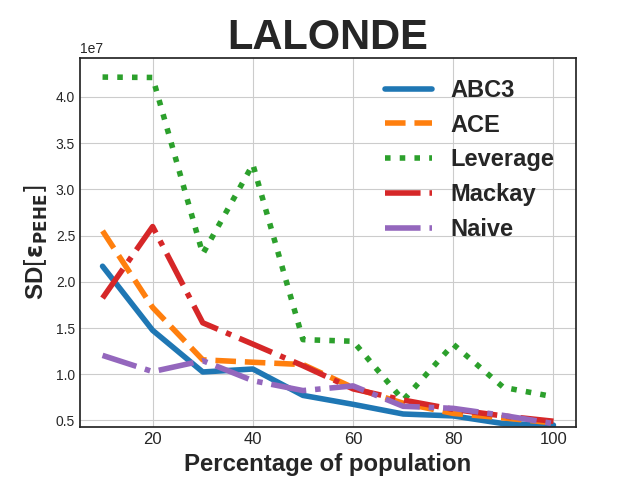}
    \caption{Standard deviation of $\epsilon_{PEHE}$.}
\end{figure*}

\section{Extending ABC3 to Large Data Set}

Though the majority of randomized experimental data is relatively small by nature, practitioners may want to apply ABC3 to the huge number of covariates.
However, computing our target quantity (Equation (1) in the main paper) for whole unseen covariates is computationally infeasible.

To bypass the problem, we showcase the extension of ABC3, which utilizes sampling and optimization techniques.
It samples N items among whole subjects and observed outcomes and maximizes the sampling version of our target quantity with optimization techniques.
We use L-BFGS as our optimizer, and uniform sampling.
We present the policy in greater detail in Algorithm 1.

To verify the empirical applicability, we use a large \textbf{Weather} data set \citep{nobert2016weather}, consisting of 100k weather information in Szeged, where we can not apply original ABC3 and other active learning policies because variance computation is computationally infeasible.
We should predict visibility with other weather-related covariates, e.g. temperature, humidity, or wind speed, utilizing only 0.1$\sim$1\% of data.
As it is not a randomized experimental data set, we set $Y^1=Y^0$.
We run 10 experiments with different train-test splits with a ratio of 9:1.
We compare ABC3 with \textbf{Naive} and \textbf{Sample} policies, where the latter does not perform optimization, but computes our target quantity within a set of sampled unseen covariates.
The result is in Figure 2.

\begin{algorithm}[h]
\caption{Extension of ABC3}
\textbf{Input}: Current time step $t$, sampling number $n$, whole covariates set $X_{\Omega}$, covariates distribution $\mathbb{P}$, previous observations $X^1_t$ and $X^0_t$, maximizing optimizer $\mathcal{O}$, tolerance $\epsilon$, kernel $k$, noise parameter $\sigma_\epsilon$\\
\textbf{Output}: $x_{t+1}, a_{t+1}$
\begin{algorithmic}[1] 
\STATE Set an initial value $x^\mathcal{O}_0=\text{feature-wise mean}(X_{\Omega} \setminus (X^1_t \bigcup X^0_t))$
\STATE Set an initial tolerance $t_0=\infty$, $i=0$
\WHILE{$t_i > \epsilon$}
\STATE Sample $n$ elements $X_n \subset X_{\Omega}$ and $\{I^a_t\}_n \subset I^a_t$ for $a \in \{0,1\}$
\STATE Compute $\mathbf{\tilde{k}}^a_{t+1}, \mathbb{K}^a_t, \mathbf{k}^a_{t,*}$ assuming $x_{t+1}= x^\mathcal{O}_i$ and  $I^a_t=\{I^a_t\}_n$ for $a \in \{0,1\}$
\STATE $x^\mathcal{O}_{i+1}=\mathcal{O}_{x_{t+1}} \left[ \text{max}_a \frac{ \int_{X_n} \left[ (\mathbf{\tilde{k}}^a_{t+1})^T \left[ \mathbb{K}^a_t + \sigma^2_\epsilon \mathbf{I} \right]^{-1} \mathbf{k}^a_{t,*}(x) - k(x_{t+1}, x)\right]^2 d\mathbb{P}(x) }{k(x_{t+1},x_{t+1}) + \sigma^2_\epsilon - (\mathbf{\tilde{k}}^a_{t+1})^T  \left[ \mathbb{K}^a_t + \sigma^2_\epsilon \mathbf{I} \right]^{-1} \mathbf{\tilde{k}}^a_{t+1}} \right]$
\STATE $t_{i+1} = ||x^\mathcal{O}_{i+1} - x^\mathcal{O}_i||$
\STATE $i = i+1$
\ENDWHILE
\STATE $x_{t+1} = \text{arg max}_{x \in X_{\Omega} \setminus (X^1_t \bigcup X^0_t)} k(x, x^\mathcal{O}_i)$
\STATE $a_{t+1} = \text{arg max}_{a} \left[ \frac{ \int_{X_n} \left[ (\mathbf{\tilde{k}}^a_{t+1})^T \left[ \mathbb{K}^a_t + \sigma^2_\epsilon \mathbf{I} \right]^{-1} \mathbf{k}^a_{t,*}(x) - k(x_{t+1}, x)\right]^2 d\mathbb{P}(x) }{k(x_{t+1},x_{t+1}) + \sigma^2_\epsilon - (\mathbf{\tilde{k}}^a_{t+1})^T  \left[ \mathbb{K}^a_t + \sigma^2_\epsilon \mathbf{I} \right]^{-1} \mathbf{\tilde{k}}^a_{t+1}} \right]$
\STATE \textbf{return} $x_{t+1}, a_{t+1}$
\end{algorithmic}
\end{algorithm}

\begin{figure}[h]
    \centering
    \includegraphics[width=0.4\linewidth]{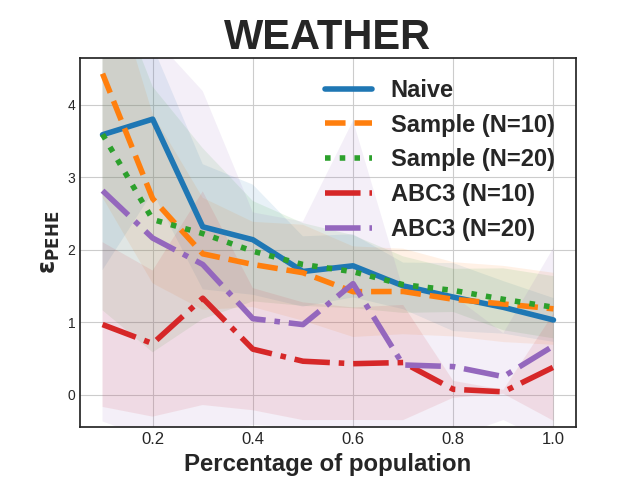}
    \caption{Mean and standard deviation of $\epsilon_{PEHE}$ on large Weather data set. $N$ is the number of utilized samples.}
\end{figure}

ABC3 substantially outperforms other policies despite small sample sizes.
Especially, the gap between ABC3 and Sample policy shows the importance of the optimization step.
Meanwhile, the performances of Sample and Naive policies are indistinguishable.
The result shows the extensibility of ABC3 to large data sets when equipped with suitable techniques.





\section{Plugging-in Other Regressor}

In our paper, we draw ABC3 by assuming our data follows the Gaussian process.
However, we can flexibly change the regressors trained on observed data sets, while maintaining ABC3 policy.
It would be beneficial if we achieve lower error by utilizing more sophisticated regressors, like neural networks.
We present $\epsilon_{PEHE}$ for different regressors.
We test Gaussian process (\textbf{GP}), random forest (\textbf{RF}, \citet{breiman2001random}), kernel support vector machine (\textbf{SVM}, \citet{cristianini2002support}), and 3 layered neural network with hidden size 100 (\textbf{NN}).
The result is in Figure 3.

\begin{figure}[h]
    \centering
    \includegraphics[width=0.4\linewidth]{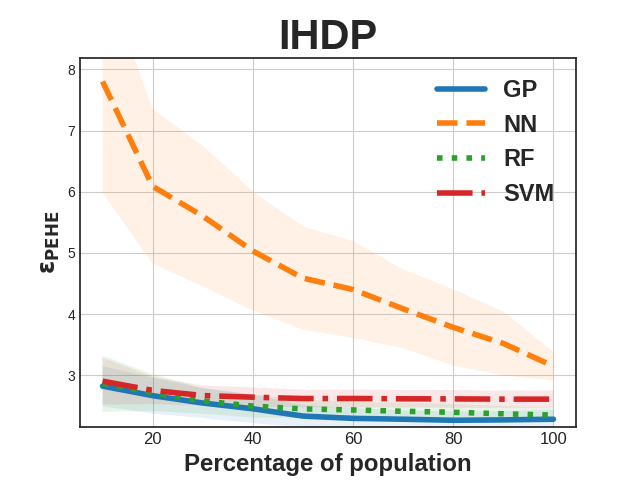}
    \includegraphics[width=0.4\linewidth]{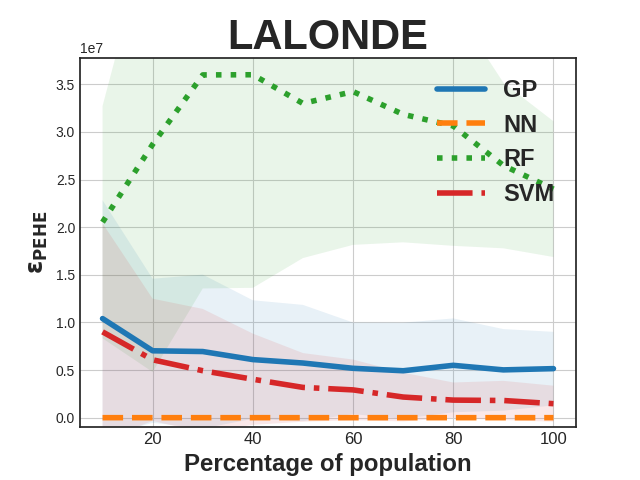}
    \caption{$\epsilon_{PEHE}$ for different regressors.}
\end{figure}

Regressors show different trends depending on data sets.
For IHDP, NN significantly deteriorates compared to other models.
SVM underperforms GP and RF for IHDP.
However, for Lalonde, NN significantly outperforms the others.
Also, SVM outperforms GP, while RF shows the worst performance.

The result shows that we can obtain better estimation with a more appropriate regressor, while it can also significantly deteriorate.
We recommend practitioners to check the suitability of regressors before applying them.

\section{Results for ACIC Data Sets}

\begin{figure*}[h]
    \centering
    \includegraphics[width=0.3\linewidth]{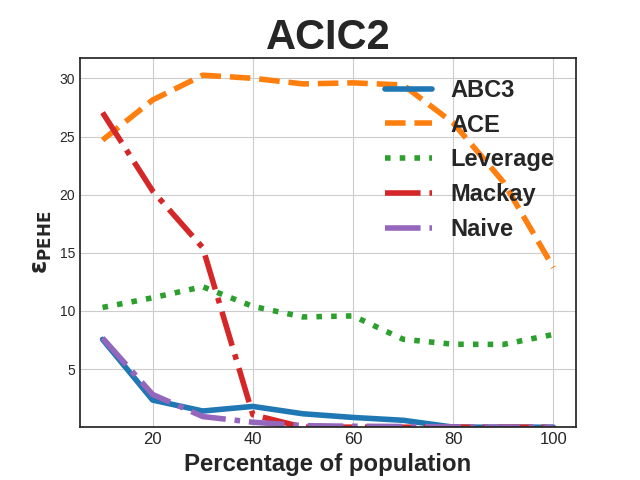}
    \includegraphics[width=0.3\linewidth]{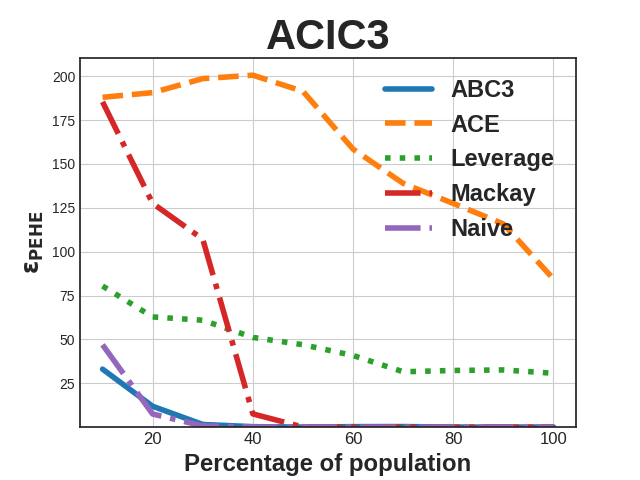}
    \includegraphics[width=0.3\linewidth]{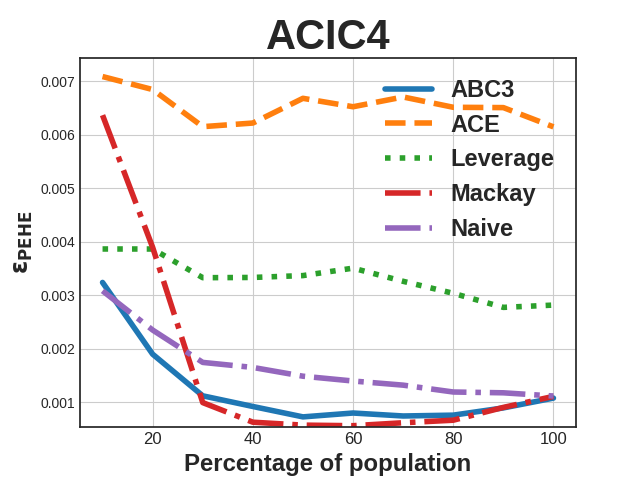}
    \caption{$\epsilon_{PEHE}$ for all ACIC test data sets (cont.).}
\end{figure*}

\begin{figure*}[h]
    \centering
    \includegraphics[width=0.245\linewidth]{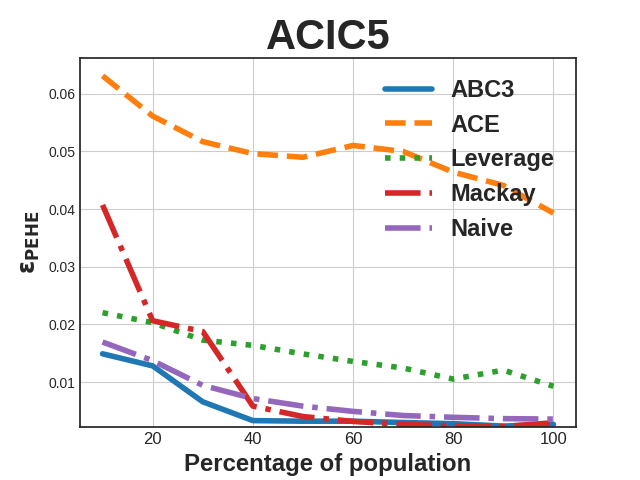}
    \includegraphics[width=0.245\linewidth]{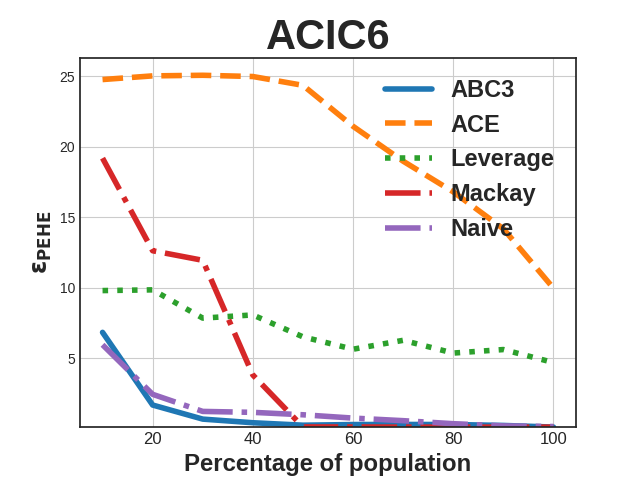}
    \includegraphics[width=0.245\linewidth]{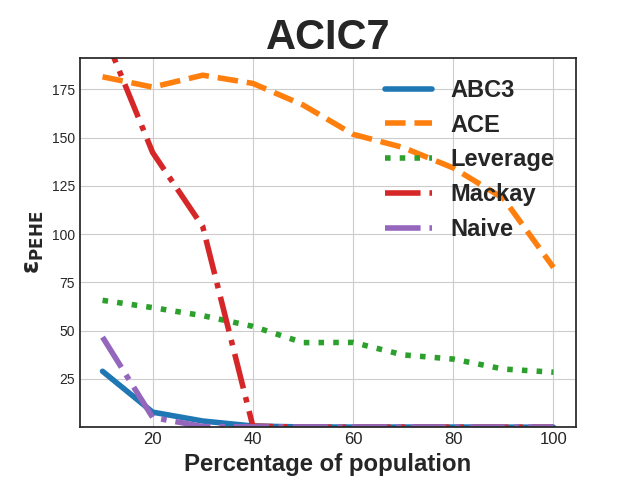}
    \includegraphics[width=0.245\linewidth]{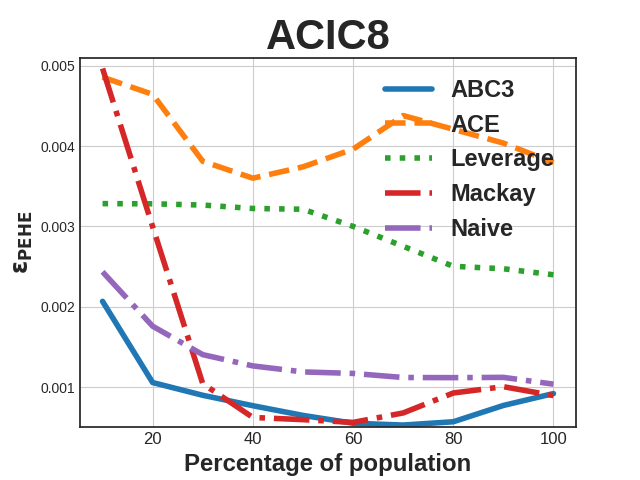}
    \caption{$\epsilon_{PEHE}$ for all ACIC test data sets.}
\end{figure*}

\end{appendices}

\bibliography{aaai25}
